\documentclass[runningheads]{llncs}
\usepackage{thm-restate}

\usepackage[T1]{fontenc}
\usepackage{comment}
\usepackage{color}
 \usepackage{longtable}

\usepackage{graphicx}
\usepackage{subcaption}
\usepackage[linesnumbered,ruled,vlined]{algorithm2e}
\usepackage{amsmath}
\usepackage{booktabs}
\usepackage{multirow}
\usepackage{array}
\usepackage{tabularx}
\usepackage{cleveref}
\usepackage{xspace}
\usepackage{wrapfig}
\usepackage[font=small,labelfont=bf]{caption}

\definecolor{gray}{rgb}{0.35,0.35,0.35}
\definecolor{blue}{rgb}{0,0,1}
\definecolor{red}{rgb}{1,0,0}
\definecolor{orange}{rgb}{0.75, 0.4, 0}
\definecolor{green}{rgb}{0.0, 0.5, 0.0}

\newcommand{\ignore}[1]{}

 \def\A{\mathcal{A}}

\newcommand{\stormr}{{\tt StoRMR}\xspace}
\newcommand{\rstormr}{{\tt R-StoRMR}\xspace}

\newcommand{\storage}{\ensuremath{W}\xspace}

\newcommand{\cols}{\ensuremath{C}\xspace}
\newcommand{\rows}{\ensuremath{R}\xspace}

\newcommand{\storearr}{\ensuremath{\A}\xspace}

\newcommand{\inc}{\ensuremath{A}\xspace}
\newcommand{\dep}{\ensuremath{D}\xspace}

\begin{document}
\title{Complete, Scalable, and Robust Prioritized Planning for Multi-Robot Ordered Storage\\ and Retrieval at Maximum Capacity%
}
\titlerunning{Prioritized Planning for Multi-Robot Storage and Retrieval}
\author{
William Zhang \and
Tzvika Geft  \and
Jingjin Yu  \and %
Kostas Bekris
}
\authorrunning{W. Zhang et al.}
\institute{
Department of Computer Science, Rutgers University, New Brunswick, NJ, USA
}
\maketitle              %

\vspace{-5mm}
\begin{abstract}
Automated warehouses face a fundamental trade-off between maximizing storage density and achieving high retrieval throughput. While puzzle-based storage (PBS) architectures increase capacity by eliminating aisles, coordinating multiple robots in these high-density spaces is computationally challenging. %
This paper formalizes the challenge through a novel multi-robot problem formulation for ordered storage and retrieval:
We consider rectangular 2D grids, where uniform-sized loads are first stored, up to full capacity, and subsequently retrieved according to prescribed arrival and departure sequences.
The main contribution of this work is an online prioritized multi-agent path planning algorithm for this problem.
The algorithm builds on prior work that constructs arrangements supporting sequential storage and retrieval, i.e., of one load at a time, without relocating loads.
By exploiting the structural invariants of such arrangements, we achieve the scalability of decoupled planning while guaranteeing complete, deadlock-free parallel execution even at full storage density.
Experiments demonstrate that the algorithm achieves near-linear improvement in makespan with respect to the number of robots, up to $C$ robots, where $C$ is the width of the grid's open side.
Furthermore, the algorithm supports robust storage arrangements that accommodate bounded uncertainty in the departure sequence, with negligible impact on execution makespan.

\end{abstract}

\section{Introduction}

\begin{wrapfigure}[11]{r}{0.35\textwidth}
\vspace{-.7in}
  \begin{center}
    \includegraphics[width=0.3\textwidth]{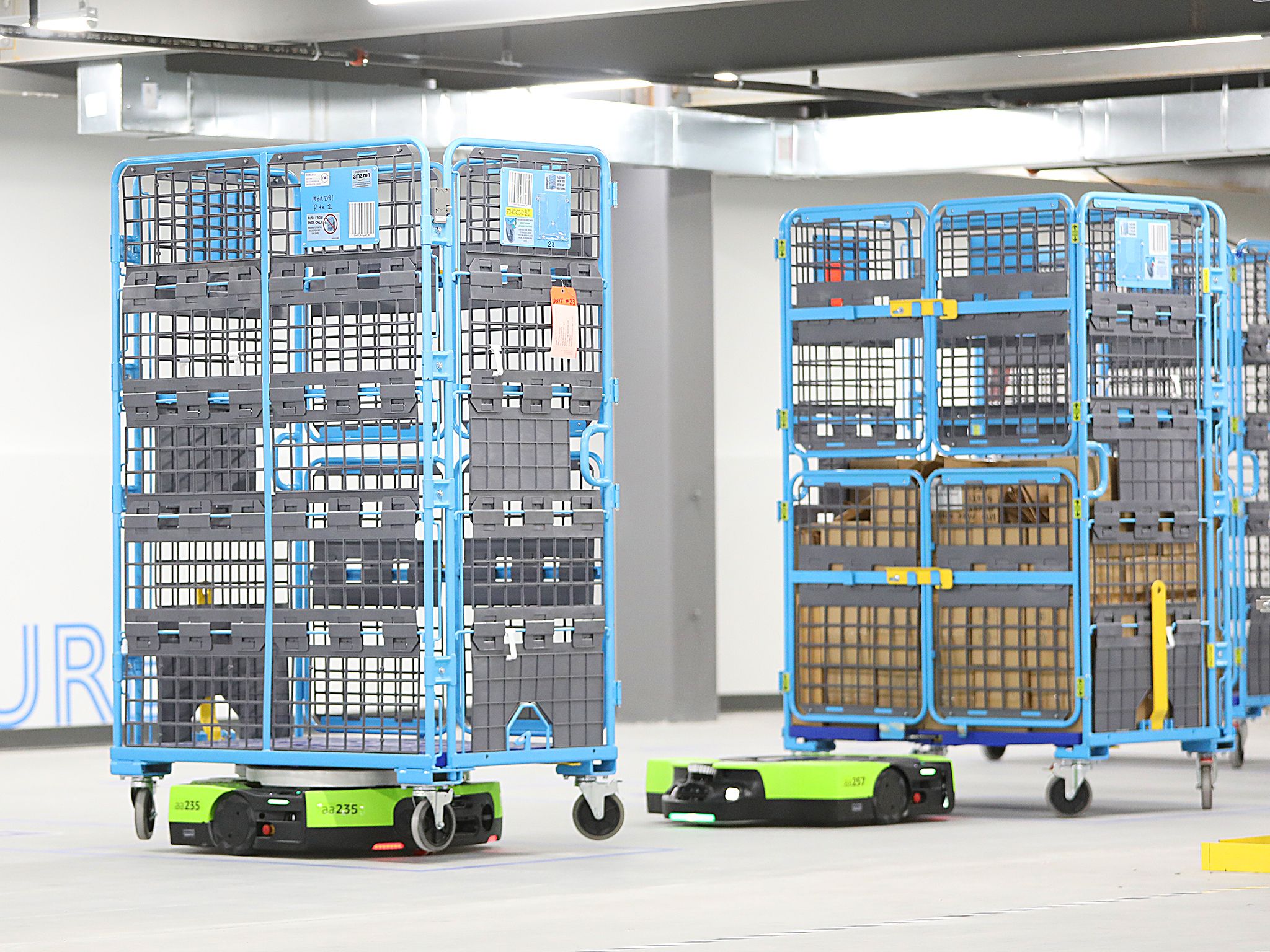}
  \end{center}
  \vspace{-.2in}
  \caption{Amazon's Proteus robots \cite{amazon_proteus} carry uniform loads and are example platforms that can be used for multi-robot ordered storage and retrieval in warehouses.}
  \vspace{-.3in}
  \label{fig:proteus}
\end{wrapfigure}

The rapid expansion of global e-commerce has placed unprecedented pressure on automated warehouses to maximize two competing, yet equally critical, metrics for return on investment: storage density and retrieval throughput. Modern facilities increasingly adopt high-density warehouse designs to reduce the physical footprint of inventory, especially as smaller facilities are being built in urban centers to expedite delivery times, where real estate costs are significant. Shelf-based Automated Storage and Retrieval Systems (AS/RS), such as Kiva-style ones~\cite{Wurman2008}, use robots like the Amazon Proteus (Fig.~\ref{fig:proteus}). While operationally efficient, they depend on dedicated aisles and buffer zones in the warehouse to facilitate robot navigation. This layout sacrifices valuable space that could otherwise be used for storage.

\begin{figure}[t]
    \centering
    \begin{subfigure}[b]{0.3\textwidth}
        \centering
        \includegraphics[width=\textwidth]{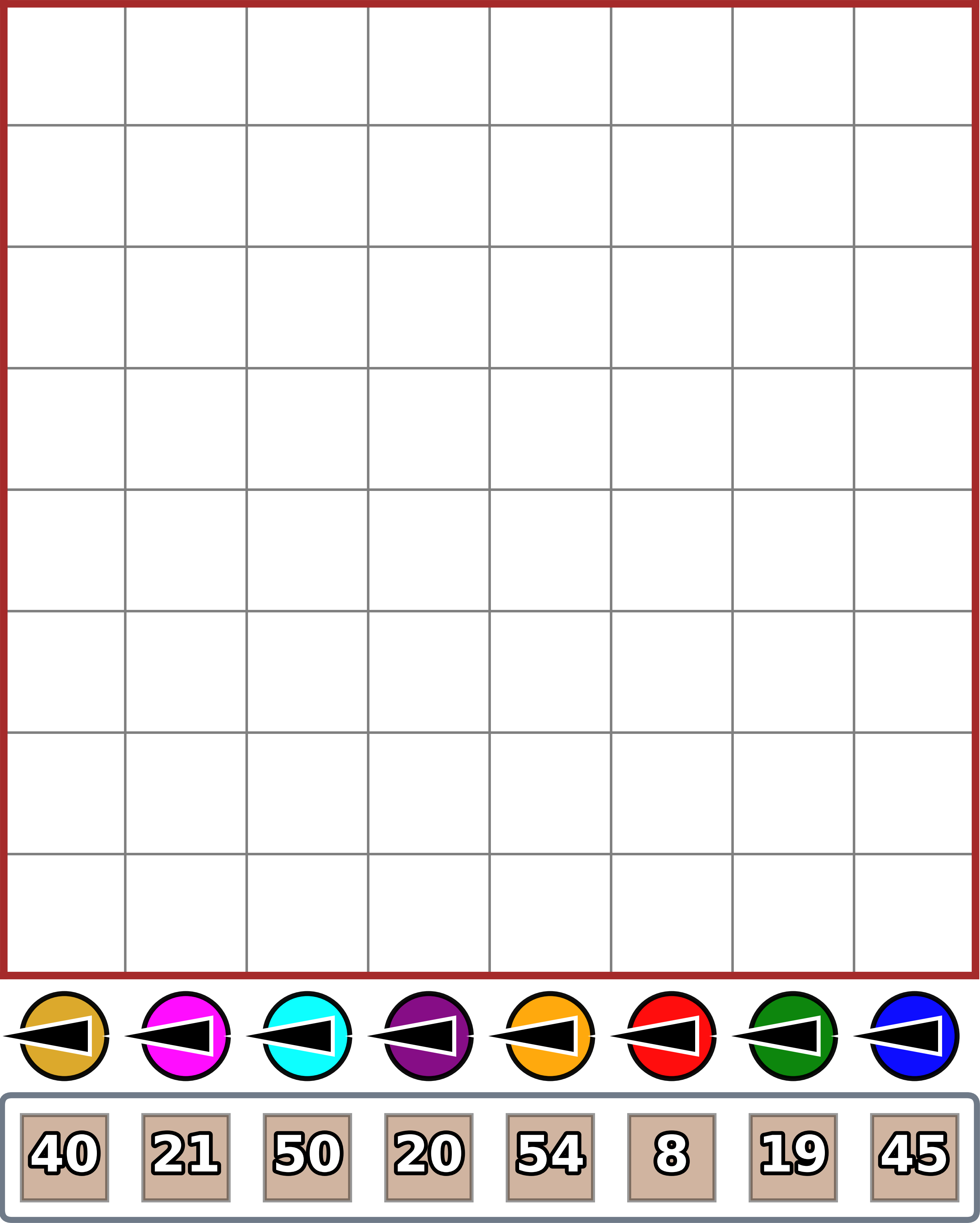}
        \label{fig:store_1}
    \end{subfigure}
    \quad
    \begin{subfigure}[b]{0.3\textwidth}
        \centering
        \includegraphics[width=\textwidth]{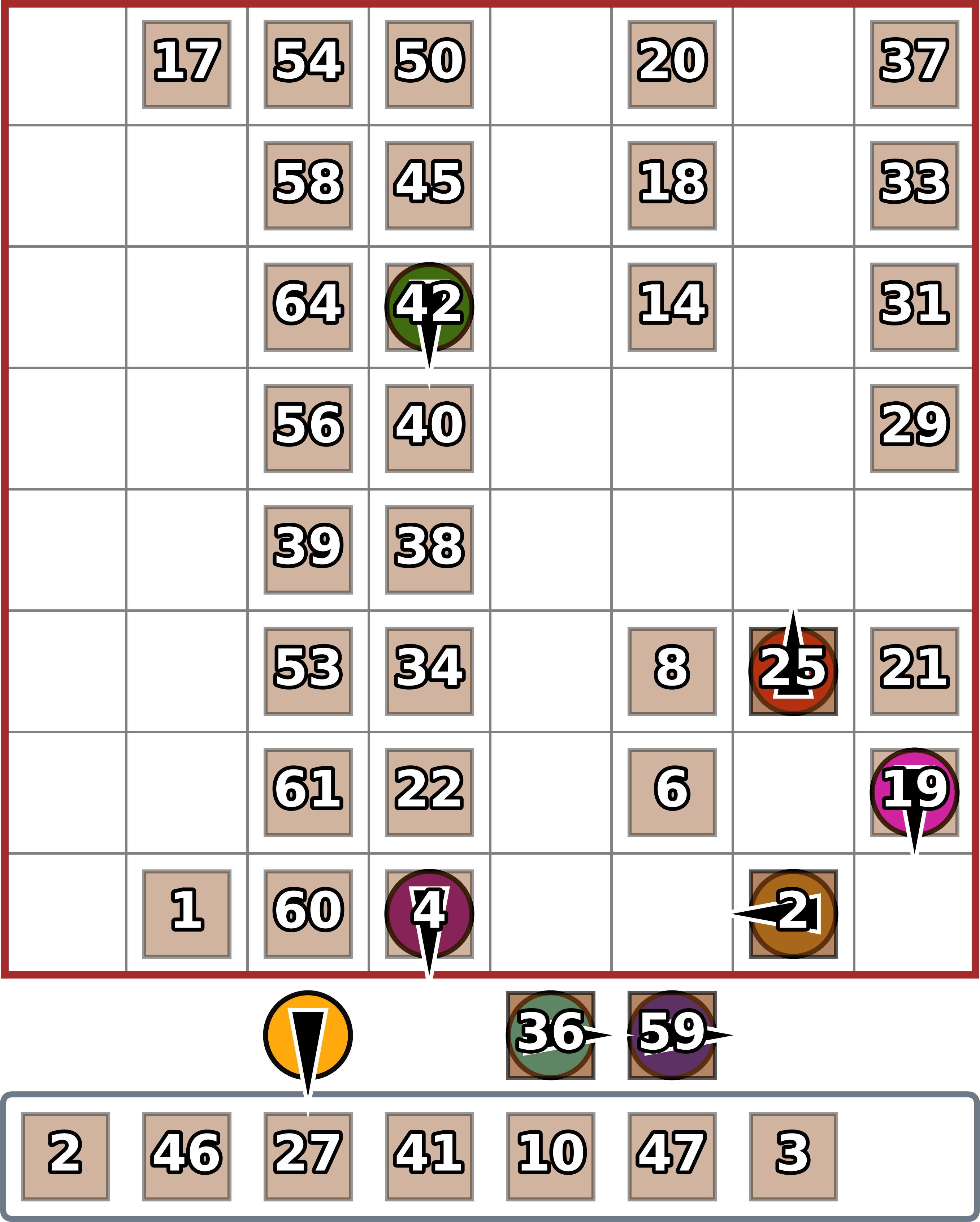}
        \label{fig:store_2}
    \end{subfigure}
    \quad
    \begin{subfigure}[b]{0.3\textwidth}
        \centering
        \includegraphics[width=\textwidth]{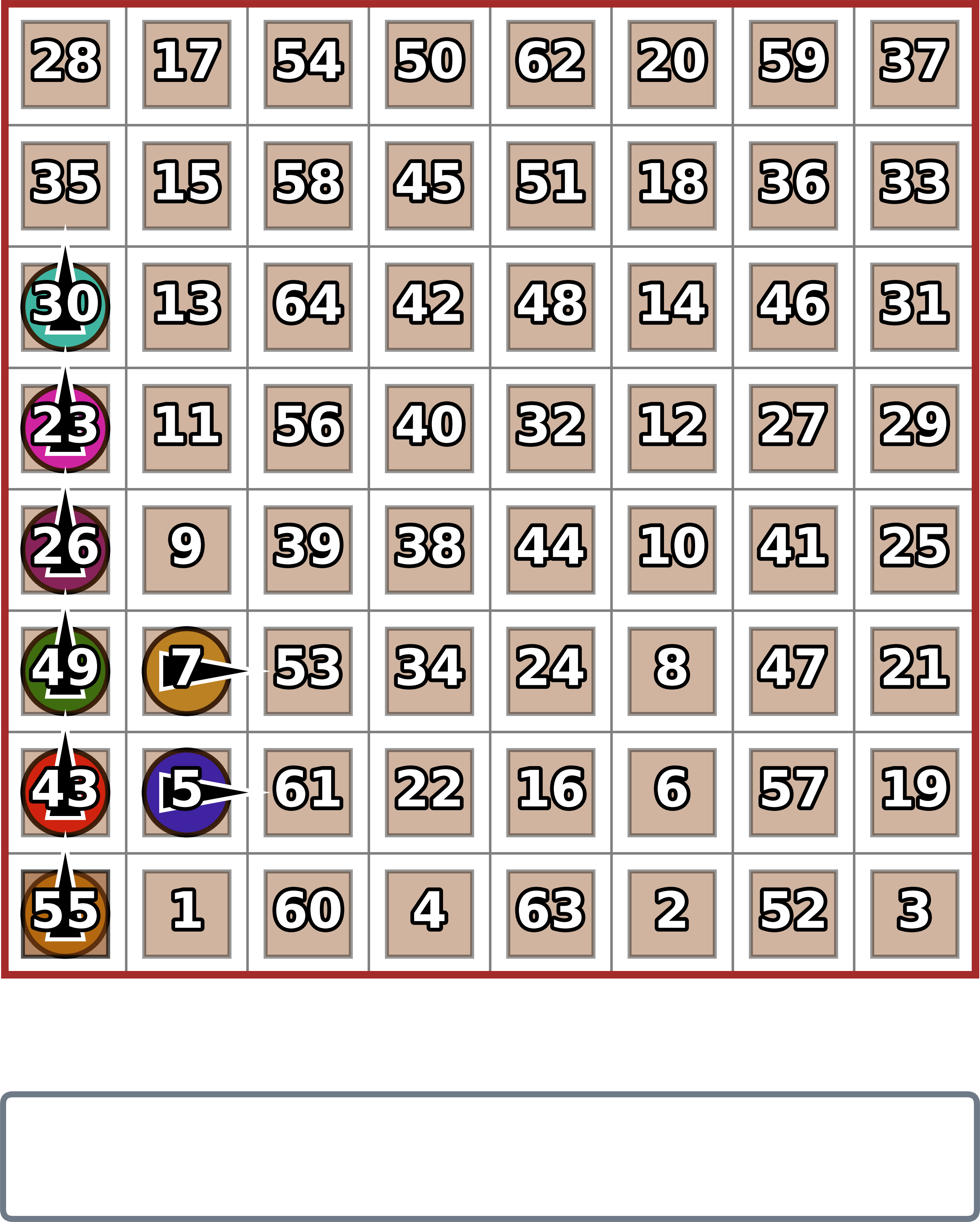}
        \label{fig:store_3}
    \end{subfigure}

    \begin{subfigure}[b]{0.3\textwidth}
        \centering
        \includegraphics[width=\textwidth]{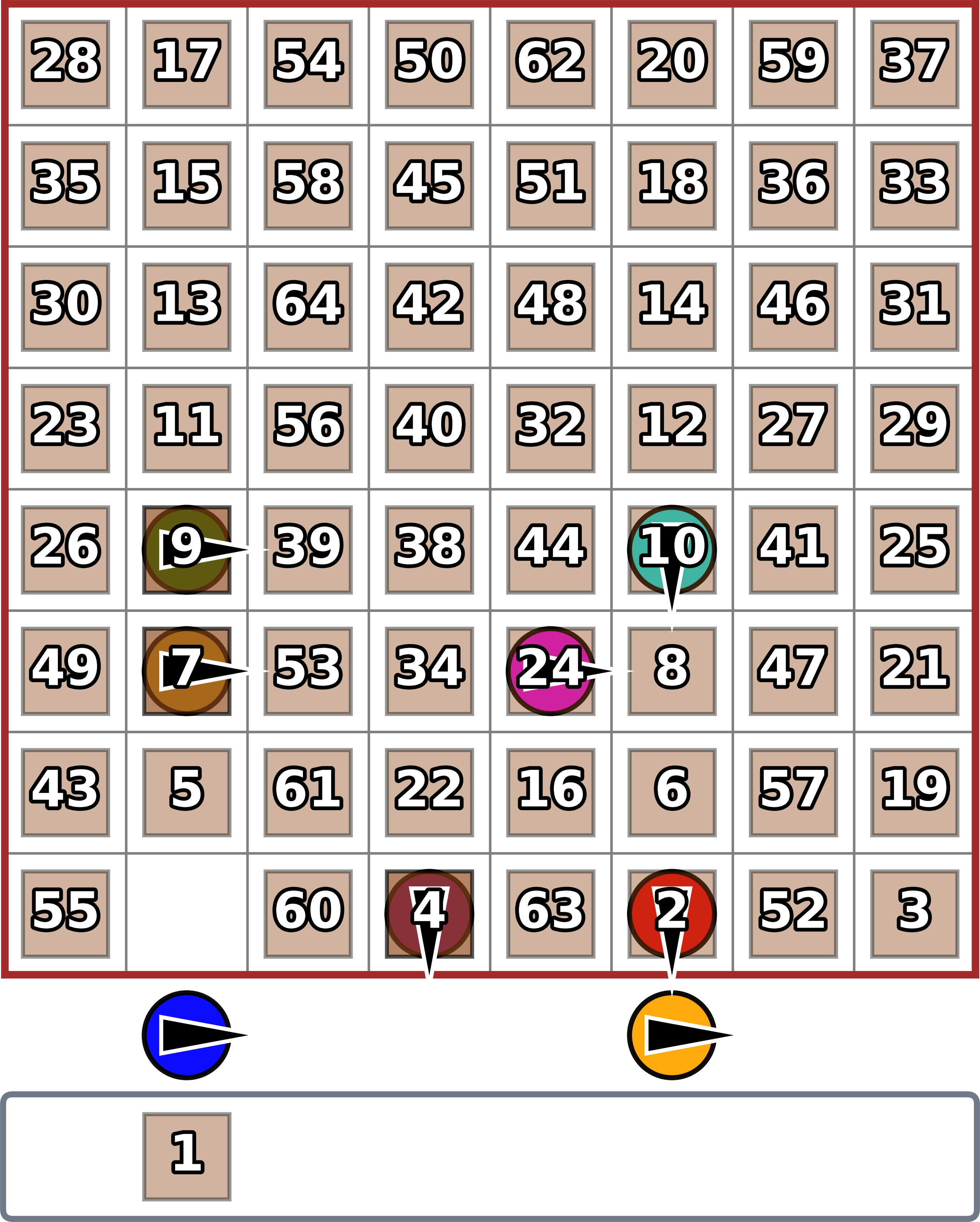}
        \label{fig:ret_1}
    \end{subfigure}
    \quad
    \begin{subfigure}[b]{0.3\textwidth}
        \centering
        \includegraphics[width=\textwidth]{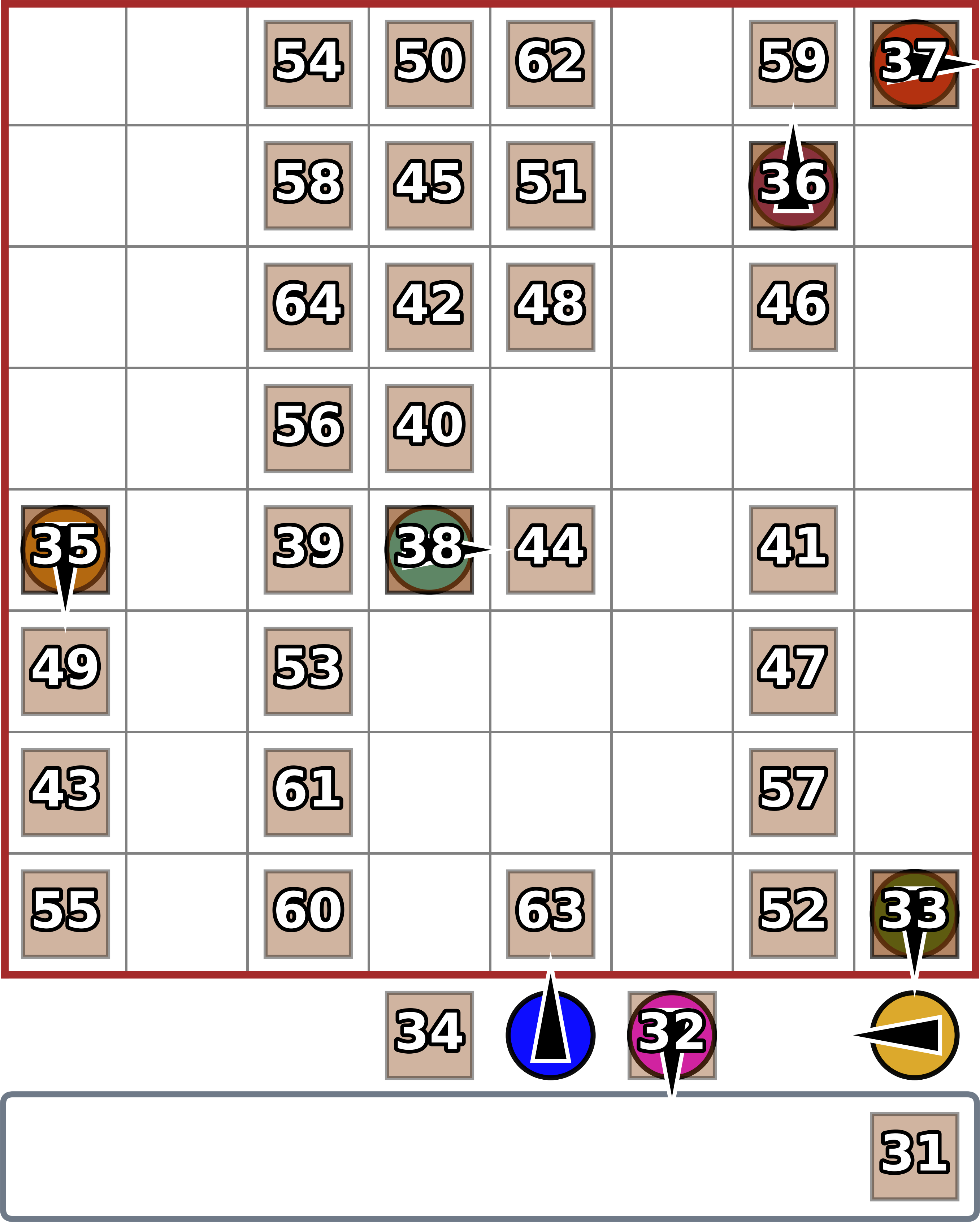}
        \label{fig:ret_2}
    \end{subfigure}
    \quad
    \begin{subfigure}[b]{0.3\textwidth}
        \centering
        \includegraphics[width=\textwidth]{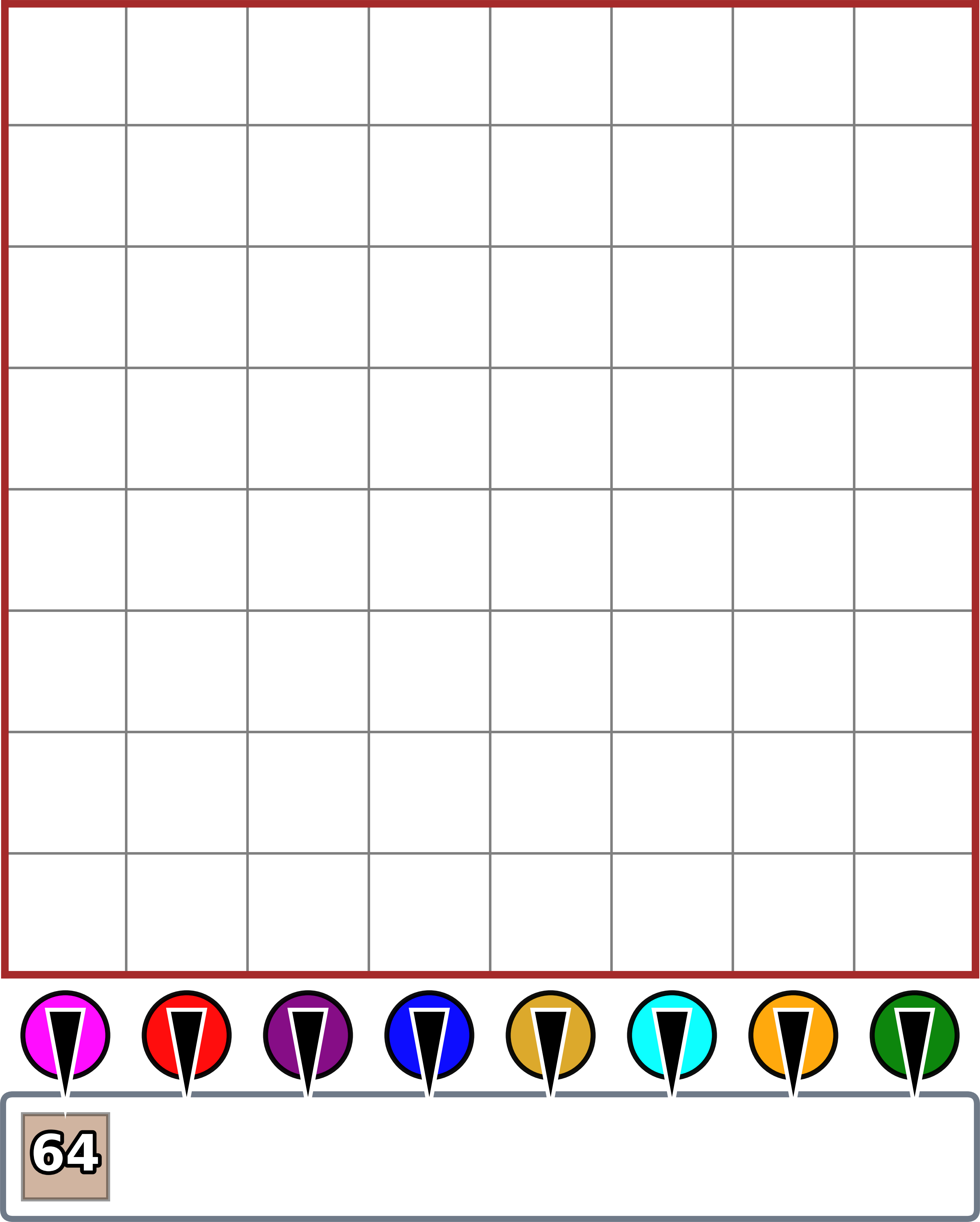}
        \label{fig:ret_3}
    \end{subfigure}

    \vspace{-.2in}

    \caption{Snapshots of a multi-robot storage and retrieval problem for a grid-based system with 8 robots. The storage area is depicted as an $8\times 8$ square. The conveyor belt is depicted as a $1\times 8$ rectangle below the storage area, where loads arrive and depart from left to right. Robots are depicted with colored circles with an arrow indicating direction. Loads are depicted with squares labeled per their departure order. Top left: an initially empty grid with an initial set of loads arriving via a conveyor belt. Top right: completion of the storage phase, with the warehouse at full capacity. Bottom left: first load retrieval. Bottom right: the final state where the last load is retrieved and moved to the conveyor belt.}
    \label{fig:intro_ex}
    \vspace{-.25in}
\end{figure}

To address these challenges, high-density storage systems have emerged, often modeled as Puzzle-Based Storage (PBS) \cite{gue2007puzzle}. This framework features a 2D rectangular grid, where each grid cell stores a uniform sized load. A load can be moved along cardinal directions via a collision-free path. In these environments, the storage grid acts like a dense sliding-tile puzzle, removing internal aisles, where loads are rearranged using a small number of empty cells to create paths for retrieval.
Coordinating multiple robots in these dense environments, however, is computationally challenging due to the high risk of deadlocks and limited maneuvering space.
This work builds on a model for grid-based storage that splits operations into two distinct phases: first, storing incoming items up to full capacity, and later, quickly retrieving them for transport. Such operations occur in facilities where goods must be stored before changing transport mode (e.g., from a ship to trucks in a port, or from a semi-truck to delivery vans in last-mile distribution centers). In these setups, it is critical to make the best use of available space and quickly serve outbound transportation vehicles as they arrive. 

Prior work, such as \stormr (Storage and Retrieval with Minimum Relocations)~\cite{stormr25}, has established the geometric feasibility of relocation-free arrangements for such grid-based storage systems. This means it is possible to store loads up to full capacity in a grid-based storage so that loads can be retrieved according to a scheduled departure sequence without any relocation, i.e., a rearrangement of other stored loads, under weak assumptions about the grid's width. This model has been extended to accommodate controlled uncertainty in the departure sequence, as in \rstormr (Robust \stormr)~\cite{aaai26}. These prior works, however, assume a sequential movement model in which only one robot moves a load at a time. Consequently, the execution-time benefits of these dense arrangements in a parallelized, multi-robot setting remain largely unexplored. 

This paper bridges the gap between relocation-minimizing storage arrangements and practical, efficient multi-robot execution as highlighted in Fig.~\ref{fig:intro_ex}. It introduces a novel multi-robot formulation for the Ordered Storage and Retrieval problem under maximum storage capacity. Then, this paper describes an online, prioritized Multi-Agent Path Finding (MAPF) algorithm to effectively coordinate multiple robots to execute the corresponding storage and retrieval sequences.
Unlike coupled, centralized planners that often time out in congested spaces due to the curse of dimensionality, the approach exploits specific layout properties to achieve efficiency via the proposed decoupled, prioritized approach. At the same time, it also guarantees completeness and prevents deadlocks even at $100\%$ load density, which is a rare case for prioritized MAPF planning methods, e.g., ~\cite{silver2005cooperative}.\\ 

\textbf{Contributions:} Overall, the primary contributions of this work are:
\vspace{-.05in}

\begin{itemize}
    \item It introduces a multi-robot formulation for ordered storage and retrieval at maximum capacity.
    \item It proposes an asynchronous, online prioritized planning algorithm that enables continuous task execution for relatively large-scale instances.
    It provides theoretical arguments that the algorithm remains complete by utilizing the geometric properties of relocation-free storage arrangements.
    \item Extensive experiments on grids of size up to $30 \times 30$  demonstrate the algorithm's \emph{efficiency and scalability}, which achieves near-linear makespan reduction as the robot count increases. The planner maintains low computational runtime per task, making it suitable for online, real-time operation. 
    \item The evaluation also shows that the cost of handling \emph{departure uncertainty} is negligible; \rstormr storage arrangements provide safety against controlled sequence perturbations without significant penalties to execution speed.
    \item In terms of \emph{solution quality}, the algorithm exhibits low makespan suboptimality when compared to a theoretically optimal (but non-scalable) centralized, coupled planner. 
\end{itemize}

\section{Related Work}

Significant progress has been made in Multi-Agent Path Finding (MAPF) and Multi-Agent Warehouse Rearrangement (MAWR). The combination, however, of efficient storage arrangements in Puzzle-Based Storage (PBS) and MAPF techniques has received less attention. In other words, the intersection of full density, strict arrival and departure sequencing, and multi-robot coordination remains a largely unexplored topic.

\textbf{High-Density Storage and Retrieval.}
High-density storage systems, often formalized as Puzzle-Based Storage (PBS), were introduced to maximize space utilization by eliminating permanent aisles. Work has established the framework for PBS by modeling the storage grid on the classic ``15-puzzle'' game~\cite{gue2007puzzle}. In this model, inventory items are treated as tiles that must navigate a dense grid using a limited number of empty cells, known as ``escorts,'' to reach an input-output point. Building on this foundation, subsequent efforts extended the analytical framework by deriving closed-form expressions for retrieval times in systems with multiple, randomly distributed escorts~\cite{kota2015}. While PBS systems often allow freedom in retrieval order~\cite{Mirzaei2017,He2023}, this flexibility does not address scenarios requiring strict sequencing.

To address sequencing constraints, the Block Relocation Problem (BRP) models stack-based storage where items must be retrieved in a predefined order. Research has proven that the BRP is NP-hard and developed algorithms to minimize the relocation of blocking items~\cite{caserta2020}. Under uncertainty, variants such as the Parallel Stack Loading Problem focus on loading phases where the precise retrieval sequence is unknown~\cite{boge}.

\textbf{Multi-Agent Path Finding (MAPF).}
MAPF is a key underlying problem in our setting, seeking collision-free paths for multiple robots between specified start and goal vertices on a graph.
MAPF is NP-hard already on 2D grids~\cite{GH2021AAMAS,MAPF-socs23}.
A popular MAPF algorithm is Conflict-Based Search (CBS)~\cite{SHARON201540}, which speeds up standard A* by coupling agents only when conflicts arise. While CBS provides strong theoretical foundations, it often suffers from scalability issues in congested environments. To mitigate this, Priority Inheritance with Backtracking (PIBT)~\cite{PiBT} introduces reactive, priority-driven mechanisms. PIBT excels in dense settings by allowing blocked agents to inherit priority from their neighbors, effectively resolving local deadlocks. Previous approaches like Push and Swap~\cite{Luna2011} guaranteed completeness in dense settings but lacked solution quality. More recently, Large Neighborhood Search (LNS) algorithms~\cite{Li2022LNS2} have emerged as an effective approach to rapidly repair suboptimal paths in congested grids.

The coordination strategy proposed in this work aligns with the principles of Lifelong Multi-Agent Path Finding~\cite{lifelong}, which enables continuous task execution by iteratively updating paths. While traditional prioritized solvers rely on environments with sufficient buffer space to resolve conflicts, the proposed approach adapts these methodologies to the unique constraints of 100\% density, a domain theoretically grounded in pebble motion on graphs~\cite{Kornhauser1984}.

\textbf{Multi-Agent Warehouse Rearrangement.}
The Block Rearrangement Problem (BRaP)~\cite{Fu2026BRaP} investigates symbolic planning for moving blocks in extremely dense grids, treating loads as active agents to simplify coordination. The setting in this work differs from BRaP by utilizing a two-level height model, akin to Kiva-style AMR systems~\cite{Wurman2008}. In the two-level architecture, autonomous robots (AMRs) navigate beneath storage loads. This introduces unique constraints where pickup and drop-off actions must be explicitly modeled to effectuate load movement, unlike the sliding-tile abstraction used in BRaP.

Recent developments have led to solvers specifically for this two-height model. The Multi-Agent Warehouse Rearrangement (MAWR) problem incorporates an initial layout, movable obstacles, and agents. The DD-MAPD framework for MAWR achieves scalability by decomposing shelf trajectories from robot tasks, though it requires ``well-formed'' environments with sufficient buffer space to guarantee completeness~\cite{Li2023DoubleDeck}. 

The MARPF approach employs an ILP-based solution to actively relocate non-target racks as movable obstacles, but it primarily focuses on path-clearing for specific targets rather than global arrangement~\cite{marpf}. NAT-CBS introduces a makespan-optimal, coupled algorithm for MAWR by integrating a high-level obstacle planner with a low-level anonymous MAPF flow network, yet it remains computationally intensive in high-occupancy settings~\cite{socs2025_sherma}. 

Recent work has explored rearrangement-free retrieval in fully packed grids for two-step storage and retrieval setups~\cite{stormr25}, including under uncertainty for the departure sequence~\cite{aaai26}. These efforts motivate the current work but do not address the multi-agent coordination aspects required for fast and effective storage and retrieval in practice.

\section{Problem Formulation}

\subsection{System Setup}

\textbf{Environment and Layout.}
We consider a rectangular storage area $W$ consisting of $R$ rows and $C$ columns of grid cells. Two special rows lie below $W$ to facilitate storage and retrieval of loads: the \textbf{I/O row} at index $R+1$, used as a temporary buffer, and the \textbf{conveyor} at index $R+2$, which supplies arriving loads and receives departing loads sequentially.
Robots are not permitted on the conveyor. The I/O row and the conveyor are carefully designed to enable lock-free multi-robot operations.

There are $n=R\cdot C$ distinct loads labeled $1, \dots, n$, and $m \le \cols$ robots. Each cell can contain at most one load and at most one robot.
Each robot can carry at most one load. An \emph{arrangement} $\storearr$ of the loads is an injective mapping that specifies a cell in $W$ (i.e., a row and a column) for each load. 

\textbf{Robot Actions and Movement.} Each robot is oriented along one of the four cardinal directions. Time is discretized, and at each timestep, a robot can perform one of the following actions: \textbf{move} (forward one cell in the robot's current orientation), \textbf{rotate} (90-degrees CW or CCW), \textbf{pick up}, or \textbf{drop off} a load in the same cell, and \textbf{wait}. 

While standard MAPF abstracts agents as points, explicitly modeling rotation captures the reality of physical robots in tight storage grids. At the same time, however, turning actions make the problem more challenging, especially in dense environments.

A load moves only while it is carried by a robot and is stationary once dropped off, except for moving to/from the conveyor (explained below).
A two-level height model is adopted in which robots not carrying a load may pass beneath stationary loads without colliding with them.

The following collisions must be avoided: (1) \textbf{Positional collisions}, where two entities of the same type (two robots or two loads) occupy the same cell in the same timestep, and (2) \textbf{Directional collisions}, where a robot enters a cell at time $t$ that was occupied by another robot at $t-1$, unless both robots are moving in the same cardinal direction (i.e., ``train'' movement is permitted, but orthogonal and swapping motions are prohibited).

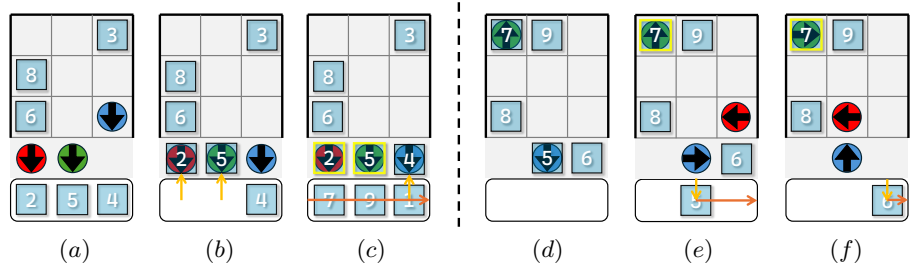
\begin{figure}[t]
\centering
\includegraphics[width=0.99\textwidth]{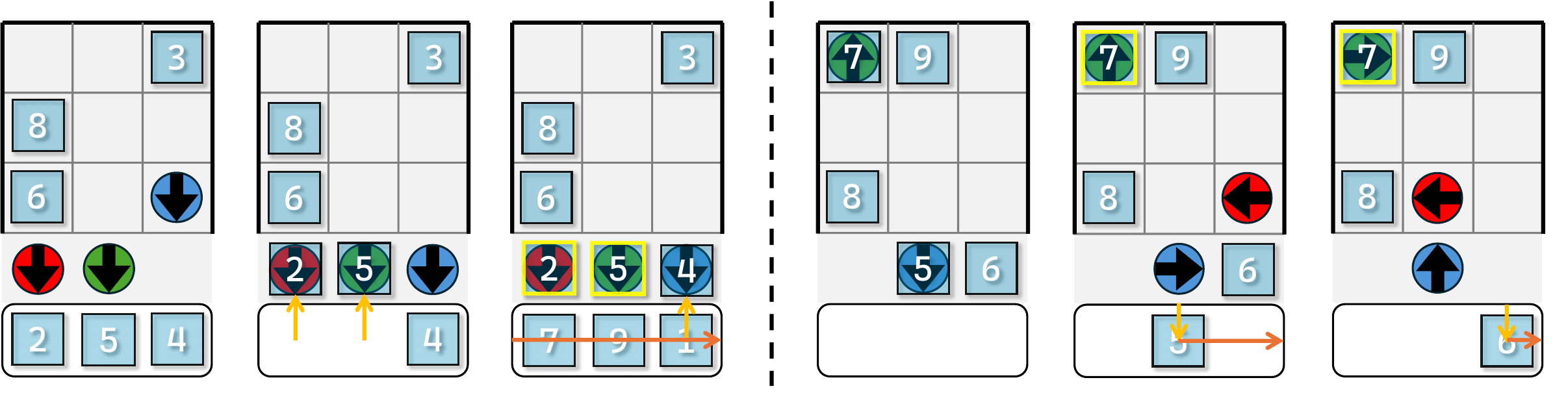}
\put(-325, -10){($a$)}
\put(-270, -10){($b$)}
\put(-213, -10){($c$)}
\put(-147, -10){($d$)}
\put(-90, -10){($e$)}
\put(-35, -10){($f$)}
\caption{Robots are represented by circles with arrows indicating direction. Loads are indicated by labeled squares, where the labels correspond to the departure sequence $D=(1,2,\dots,9)$ without loss of generality. Yellow borders indicate a load currently held by a robot. The left sequence shows three frames that highlight conveyor dynamics during the storage phase for the arrival sequence $\inc=(8,3,6,4,5,2,1,9,7)$. Frame (b) shows the conveyor pushing $2$ and $5$ to the I/O row in one timestep. Frame (c) shows a conveyor pushing $4$ to the I/O row, followed by a subsequent batch of $(1,9,7)$ moved into the conveyor row. During the retrieval phase, frame (d) shows loads $5$ and $6$ on the I/O row. In (e), load $5$ is first received by the conveyor given the departure order \dep and instantly departs to the right. Then load $6$ follows (f).}
\label{fig:warehouse_layout}
\end{figure}
\subsection{Problem Definition}

\textbf{Operational Phases.}
The system operates in two phases.
In the \textbf{storage phase}, loads arrive via the conveyor in ordered batches of size $C$ (one load per column) according to the arrival sequence $\inc=(a_1,\ldots,a_n)$.
To decouple the conveyor's motion from the robots' motion, assume that the conveyor operates as an independent unit (see Fig.~\ref{fig:warehouse_layout}). If the entire conveyor row is empty, the conveyor places the next batch of $C$ loads from left to right on the conveyor row instantly. This batching structure exposes multiple loads simultaneously, enabling the robots to access them in parallel. The conveyor pushes any load on it onto the I/O row in one timestep for robot pickup. In the \textbf{retrieval phase}, loads can be delivered to the I/O row for departure. The conveyor can receive loads from the I/O row in one timestep. The conveyor strictly follows the departure order \dep. Without loss of generality, $\dep$ is defined by $\dep=(1, 2, \dots, n)$.

A \emph{plan} $\Pi$ is a sequence of joint actions $\langle \mathbf{a}_t, c_t \rangle$ for $t = 1, \dots, T$, where $\mathbf{a}_t = \{u_{1,t}, \dots, u_{m,t}\}$ is the set of actions performed by the robots and $c_t$ denotes the conveyor action.
A plan is \emph{collision-free} if its actions avoid motion with collisions as defined above.
The \emph{makespan} of a plan is the total number of timesteps $T$.

\textbf{Formal Problem Statement.}
The \textbf{storage problem} is defined by the tuple $\langle W, \mathcal{R}, \storearr, \inc \rangle$, where $\mathcal{R}$ is the initial robot configuration, $\storearr$ is the target storage arrangement, and $\inc$ is the arrival sequence. The \textbf{retrieval problem} is defined by $\langle W, \mathcal{R}, \storearr_{\text{init}}, \dep \rangle$, where $\storearr_{\text{init}}$ is the load arrangement after storage at full capacity and $\dep$ is the departure sequence.

In both cases, the output is a collision-free plan $\Pi$ that transforms the system to the target configuration. The plan can be generated offline or executed online. The objective is to minimize the makespan $T$.

\section{Preliminaries}
\label{sec:preliminaries}

The proposed multi-robot approach builds upon the guarantees provided by sequential relocation-free storage arrangements. An arrangement \storearr is an injective mapping $\{1,\dots , n\} \rightarrow \{1, \dots , \rows\} \times \{1,\dots , \cols\}$ assigning each load to a unique grid cell~\cite{stormr25,aaai26}. The notation $\storearr[\ell]$ denotes the cell to which $\ell$ is mapped.

\begin{definition}
An arrangement $\storearr$ satisfies an arrival (or departure) sequence if every load in that sequence can be stored (or retrieved) sequentially without moving any other load, i.e., without relocations.    
\end{definition}

\begin{figure}[t]
\centering
\includegraphics[width=0.8\textwidth]{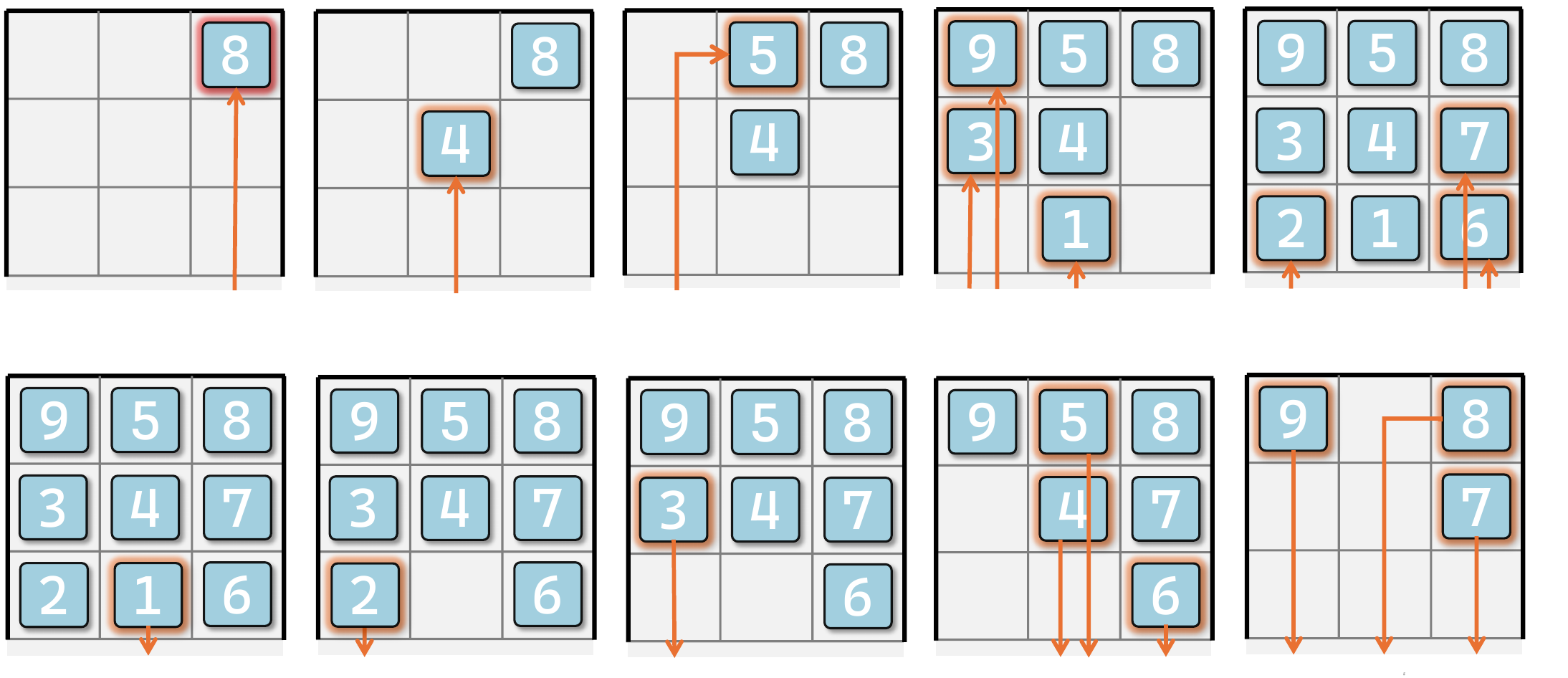}
\caption{An arrangement that satisfies $\inc=(8,4,5,9, 3,1,7,6,2)$ and $\dep=(1,2,\ldots,9)$.
The storage sequence is shown at the top, and the retrieval sequence is shown at the bottom (both from left to right).
Each load is stored or retrieved sequentially in the specified order via a collision-free path. The figure is based on~\cite{stormr25,aaai26}.} 
\label{fig:arrangement}
\end{figure}

Geometrically, satisfiability implies the existence of a collision-free path from the I/O row to the target cell (or vice-versa), avoiding specific subsets of loads:
\begin{itemize}
    \item \textbf{Storage:} For an arrival sequence $\inc=(a_1,\dots, a_n)$, the path for storing load $a_j$ from the I/O row must be collision-free given all previously placed loads $\{a_1, \dots, a_{j-1}\}$.
    \item \textbf{Retrieval:} For a departure sequence $\dep=(d_1,\dots, d_n)$, the path for retrieving load $d_j$ and moving it to the I/O must be collision-free given all loads scheduled for later departure $\{d_{j+1}, \dots, d_n\}$.
\end{itemize}
Figure~\ref{fig:arrangement} provides an example of an arrangement that satisfies \inc and \dep. Note that loads can be stored to or retrieved from any location of the I/O row. These definitions establish \emph{static feasibility}: they guarantee that valid paths exist if loads are moved one at a time, independent of conveyor dynamics or robot parallelization. Both of these aspects are the focus of the current work.

The following define parameters for departure uncertainty. 
To account for uncertainty in the departure order, the notion of $k$-robustness is useful:

\begin{definition}
Let $\tilde{S}$ be a permutation of sequence $S = (s_1, \dots, s_n)$. Then, $\tilde{S}$ is a $k$-bounded perturbation if for every pair $(s_i, s_j)$ where $i < j$ but $s_j$ appears before $s_i$ in $\tilde{S}$ (i.e., an order inversion), the index distance satisfies $j - i \le k$.
\end{definition}

\begin{definition}[$k$-Robustness]
    An arrangement is \textbf{$k$-robust} with respect to $S$ if it satisfies every valid $k$-bounded perturbation $\tilde{S}$ of $S$.
\end{definition}

The following sections include the main contribution of this work, which take storage arrangements designed for sequential motion as inputs.
The multi-robot planning algorithms are designed to succeed on any input arrangement \storearr that satisfies \inc and \dep.
The proposed algorithms leverage these invariants to guarantee completeness (Section~\ref{completeness}) while parallelizing operations for maximum throughput.

\section{Prioritized Planning Algorithm}

The proposed algorithm utilizes an asynchronous prioritized planning strategy where robots are assigned tasks and compute paths only upon becoming idle after the completion of prior tasks. Collision-free coordination is maintained via a global reservation table, which tracks space-time constraints as tuples $(p,t,d)$, which represent the position $p$, timestep $t$, and prohibited entry direction $d$, to explicitly prevent directional following conflicts.

To support the two-level height model, the reservation system tracks obstacles on two distinct layers:
\emph{Robot Layer:} Active robots reserve cells along their planned trajectories. Once a robot completes its path, it is treated as a static obstacle at its final cell for all future timesteps until assigned a new task. \emph{Load Layer:} Stored loads are marked as static obstacles. The status of load obstacles is updated given robot paths: a load-obstacle is removed from the table at the timestep a robot is planned to execute a \textit{pickup}, and a new load-obstacle appears at the timestep a robot is planned to perform a \textit{drop-off}.

\begin{figure}[t]
\centering
\includegraphics[width=0.6\textwidth]
{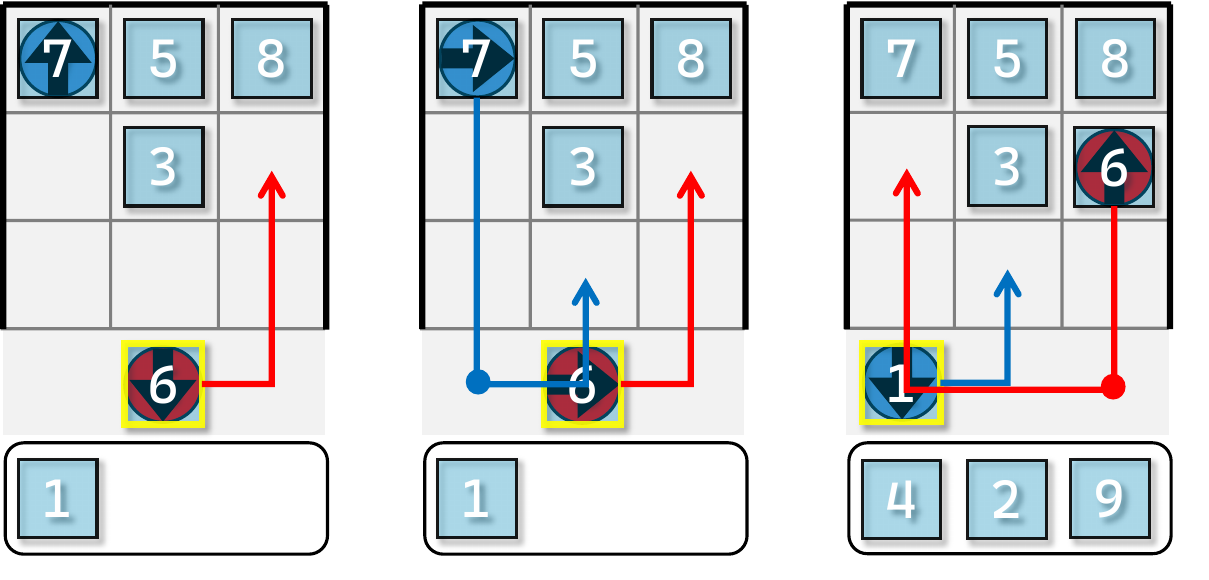}
\put(-185, -10){($a$)}
\put(-112, -10){($b$)}
\put(-40, -10){($c$)}
\caption{An example instance of the storage algorithm with $\inc=(3, 7, 5, 8, 6, 1, 9, 2, 4)$. Loads being carried are highlighted in yellow. (a) The red robot has a reserved path (in red), while the blue robot has just stored 7 and is now idle. (b) The solver plans a path for the blue robot, which includes the pickup point marked with a blue circle. (c) A few timesteps later, the red robot has finished its path, while the blue robot has a small segment remaining to drop load 1. The conveyor loads the next batch $(9, 2, 4)$ according to the arrival sequence. The red robot plans a path to pick up the next load $9$, while respecting the prioritized blue path.}
\label{fig:storage_example}
\end{figure}

\subsection{Storage}
Consider an arrival sequence $\inc=(a_1,a_2,\dots,a_n)$, and a storage arrangement \storearr that satisfies \inc. Assume planning is performed online, where robots claim tasks dynamically and do not possess knowledge of a load's destination until assignment.

\begin{algorithm}[h!]
\caption{Multi-Robot Prioritized Planning for Storage}
\SetAlgoVlined
\DontPrintSemicolon

\SetKwFunction{Plan}{planPath}
\SetKwFunction{Spot}{pickupSpot}
\SetKwFunction{Queue}{queuePath}

\KwIn{Robots $\mathcal{R}$, Arrival order $A$, Planned arrangement \storearr}
\KwOut{Storage execution plan $\Pi$}

\While{$\inc$ has unclaimed loads \textbf{or} active robots exist}{
    
    $R_{idle} \leftarrow \{ r \in \mathcal{R} \mid r\text{.queue is empty} \}$\;
    
    \While{$R_{idle}$ is not empty \textbf{and} $\inc$ has unclaimed loads}{
        $\ell \leftarrow$ next unclaimed load in $\inc$\;
        
        $r \leftarrow \arg\min_{r' \in R_{idle}} \text{dist}(r', \Spot(\ell))$\;

        $p_{start} \leftarrow \Spot(\ell)$\;
        \label{pickup}
        $p_{goal} \leftarrow \storearr[\ell]$\;
        
        $\pi \leftarrow$ \Plan{$r$, $p_{start}$, $p_{goal}$}\;
        \label{alg:a_star}
        
        \If{$\pi \text{ is success}$}{
            Mark $\ell$ as \textbf{claimed}\;
            $r.\Queue(\pi)$\;
        }
        $R_{idle} \leftarrow R_{idle} \setminus \{r\}$\;
    }
    Execute next step for all robots\;
}
\end{algorithm}

\emph{Task Assignment.} Whenever a robot becomes idle, it is considered for assignment to the next unclaimed load $a_k$ in the arrival sequence. If multiple robots are idle simultaneously, $a_k$ is greedily assigned to the robot closest to the pickup cell.

\emph{Pickup and Planning.} The pickup position for $a_k$ is located on the I/O row at the column corresponding to its batch index given by \inc (line~\ref{pickup} of algorithm 1). The planner executes a space-time A* search to generate a time-minimal trajectory from the robot's current position to the pickup cell and subsequent storage location $\storearr[a_k]$. This search checks against the global reservation table to ensure collision-free movement. If the A* search fails to find a valid path in line~\ref{alg:a_star}, the robot waits and replans in the next timestep without claiming a load, allowing another robot to potentially successfully plan a path for the load.

\emph{Conveyor Synchronization.} The planner implicitly coordinates with the conveyor. If a valid path allows the robot to arrive at the pickup cell at time $t$, the system reserves the I/O slot at time $t-1$ for the conveyor to push the load from the conveyor row. If the I/O slot is blocked at $t-1$ (e.g., by another robot), the A* search treats $t$ as invalid for pickup and searches for a later arrival time. See Fig.~\ref{fig:storage_example} for a storage sequence example.

\begin{figure}[t]
\centering
\includegraphics[width=0.6\textwidth]
{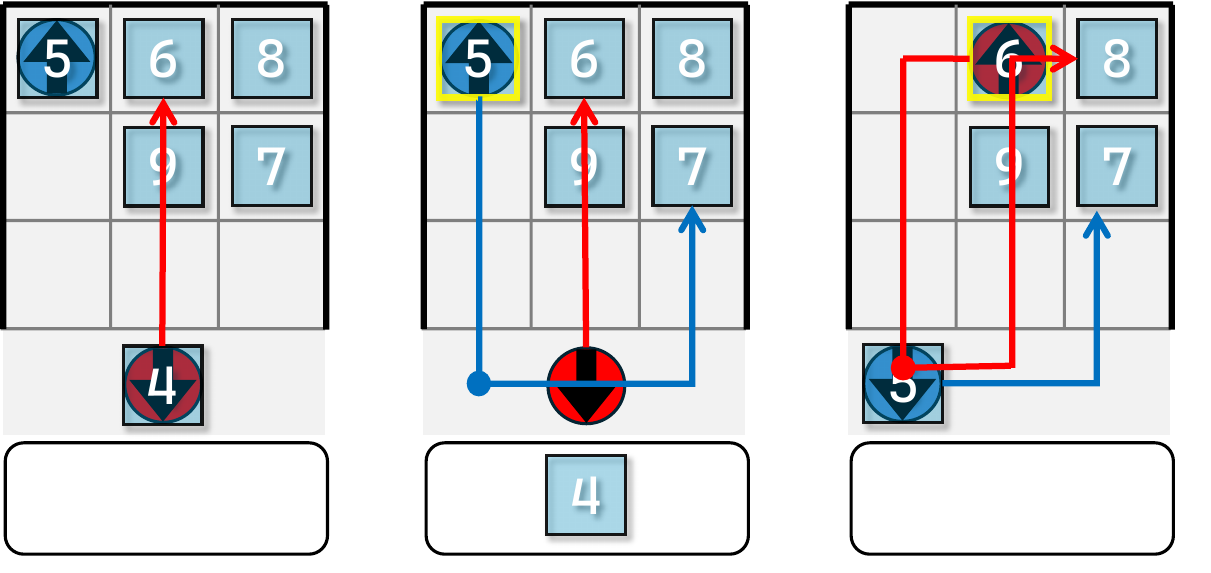}
\put(-185, -10){($a$)}
\put(-112, -10){($b$)}
\put(-40, -10){($c$)}
\caption{An example instance of the retrieval algorithm with $\dep=(1,2,3,4,5,6,7,8,9)$. Loads being carried are highlighted in yellow. (a) The red robot has an already reserved path in red, while the blue robot has just reached $5$ and is now idle. (b) The solver plans a path for the blue robot, which includes the drop-off point marked with the blue circle and the next unclaimed load in \dep, load $7$. The conveyor can also receive the next load in \dep, which is load $4$ from the red robot. (c) A few timesteps later, the red robot has finished its path, while the blue robot has a segment remaining. The red robot plans a path to drop off the current load, $6$, and then move to the next unclaimed load, $8$. The leftmost drop-off spot is reserved until the blue robot has finished its drop-off.}
\label{fig:retrieval_example}
\vspace{-10pt}
\end{figure}

\subsection{Retrieval}

The objective is to maintain the same approach as for storage but mirrored.
Each robot handles the retrieval of a load and then goes back to the storage area to wait (idle) under a stored load, to avoid interfering with subsequent retrievals.  
Applying this approach for retrieval introduces a new challenge, namely the choice of which load to have a robot wait under after it drops a load off at the I/O row. 
This choice is easy for storage, since a robot can stay under the load it just stored. 
For retrieval, robots should be positioned under loads according to the departure sequence, but there is no guarantee that a load that is yet to depart can be accessed (e.g., the 5th load to depart can be blocked from all sides by previously departing loads claimed by occupying robots). 
Furthermore, if knowledge of the departure sequence is limited, then the approach is not clearly defined. 
Therefore, the approach allows the possibility of sending a robot back under an arbitrary load.

\begin{algorithm}[t!]
\caption{Multi-Robot Prioritized Planning For Retrieval}
\SetAlgoVlined
\DontPrintSemicolon
\SetKwFunction{Plan}{PlanPath}
\SetKwFunction{Queue}{queuePath}

\KwIn{Robots $\mathcal{R}$, Departure order $\dep$, Storage arrangement $\storearr$}
\KwOut{Retrieval execution plan $\Pi$}

\While{$\dep$ has unclaimed loads \textbf{or} active robots exist}{
    \For{$r \in \mathcal{R}$ \textbf{where} $r.\text{queue}$ is empty }
    {
    \label{idle_robots_retrieval}
        $\ell \leftarrow \text{next unclaimed load in } \dep$\;
        $p \leftarrow \storearr[\ell]$\;
        \eIf{$r\text{ has claimed a load}$}
        {
            \If{$\ell-1$ has not queued a path to I/O row}
            {
                \textbf{continue}\;
            }
            \label{dep_sequencing}
            $\pi \leftarrow$ \Plan{$r$, I/O row, $p$}\;
            \label{alg:a_star_retrieval}
            \eIf{$\pi \text{ is success}$}{
            Claim $\ell$\;
            $r.\Queue(\pi)$\;
        }{
            \ForEach{next closest unclaimed load $\ell' \in \storearr$}{
            \label{under_unclaimed_load}
                $p' \leftarrow \storearr[\ell']$\;
                $\pi \leftarrow$ \Plan{$r$, I/O row, $p'$}\;
                \If{$\pi \text{ is success}$}{
                    $r.\Queue(\pi)$\;
                    \textbf{continue}\;
                }   
            }
            \label{back_row}
            $\pi \leftarrow$ \Plan{$r$, I/O row, back row}\;
        }
        }{
            
            $\pi \leftarrow$ \Plan{$r$, $p$}\;
            \label{not_carrying}
            \If{$\pi \text{ is success}$}{
            Claim $\ell$\;
            $r.\Queue(\pi)$\;
        }
        }

    }
    \ Execute next step for all robots\;
}
\end{algorithm}

Algorithm 2 details the retrieval algorithm. The input is a departure sequence, $\dep=(1,2,\dots,n)$ and an initial storage arrangement $\storearr$, which satisfies \dep. Robots avoid task conflict through a prioritized claiming system. A robot irrevocably claims the next load in the departure sequence only upon the successful calculation of a valid path. This ensures that it can transition to the next task immediately after its current delivery and wait safely under its next claimed load. 
At the start, no robots have claimed loads, so each robot finds a path to the next unclaimed departing load in \dep, namely load $j$ (line~\ref{not_carrying}). 
If a robot has claimed a load and finished its path to the load's position, $\storearr[j]$, it plans a minimum-time path to any drop-off location on the I/O row and then to its next assigned load, which is the next unassigned departing load in \dep (line~\ref{alg:a_star_retrieval}).
Each path will be reserved in the form of obstacles for future path planning, as in the storage algorithm. See Fig.~\ref{fig:retrieval_example} for an example sequence.

If planning fails in line~\ref{alg:a_star_retrieval}, the robot executes one of two fallback options. First, the robot attempts to plan a path to drop the load on the I/O row then move under the closest accessible unclaimed load (line~\ref{under_unclaimed_load}). This prevents subsequent retrieval paths from being blocked (see \Cref{completeness}). 
Alternatively, if there is no available load to wait under, the robot plans a path to drop the load on the I/O row and then moves to an empty cell in the back row to avoid obstructing the paths of other robots (line~\ref{back_row}). 
A non-obstructing cell on the furthest back row is guaranteed to exist as shown in the extended proof in~\cite[Appendix]{complete}. 
If planning still fails, the robot waits and replans in the next timestep. 
If a robot executes a fallback as described above due to a planning failure at line~\ref{alg:a_star_retrieval}, the robot completes its delivery without claiming a subsequent load. Consequently, upon finishing its path, the robot has no claimed load and re-enters the planning phase at line~\ref{not_carrying} to claim and move toward the next available load. 

\textit{Departure sequencing.}
Since the storage \storearr is guaranteed to satisfy \dep but not necessarily an arbitrary sequence, the approach enforces the departure order by ensuring a robot  plans a path only for load $j$ toward the I/O row when the previous load $j-1$ in the departure sequence has a queued path to the I/O row (line~\ref{dep_sequencing} of algorithm 2). Then, each retrieval path reserves its drop-off cell on the I/O row until the conveyor receives the load. As the conveyor accepts loads strictly in departure order, this reservation prevents out-of-sequence loads from prematurely occupying the I/O row and blocking access for the current target load. Section~\ref{completeness} uses this strict departure sequence to provide completeness arguments. To enforce this order for idle robots, line~\ref{idle_robots_retrieval} processes robots in a sorted manner. Robots with claimed loads take precedence, ordered by the departure index of their load. Robots without a load are then ordered by their distance to the next unclaimed load, following the same logic used for storage.

\vspace{-3pt}
\section{Properties}
\vspace{-3pt}
\label{completeness}

This section establishes the completeness of the multi-robot planning algorithms.

\begin{theorem}
The storage algorithm
is complete.
\end{theorem}

\begin{proof}
Let $\inc=(a_1, a_2, \dots, a_n)$ be the arrival sequence and $\storearr$ be a storage arrangement that satisfies \inc. Assume that the robots are all initially positioned on the front row of $W$ (the row above the I/O row). This allows each robot to initially have an unobstructed path to the I/O row. The proof is by induction on $k$ that load $a_k$ can be stored.

For the base case, consider the first load $a_1$. 
There exists a path $\pi$ from the I/O row to the target $\storearr[a_1]$, which passes through at most one cell with a robot $r$. If such a robot $r$ exists, then path planning for $r$ succeeds: $r$ can move via $\pi$ to pick up the load and then move it to $\storearr[a_1]$. 
If $\pi$ does not pass through an occupied cell, path planning succeeds for any robot.

Assume loads $a_1, \dots, a_{k-1}$ are already stored.  Consider the arrival of load $a_k$. Suppose there is a time step at which all the robots are idle. If planning succeeds for $a_k$ before that, then the point is shown. Otherwise, the following arguments show that it will also be true.
By the definition of $\storearr$, there exists a path $\pi$ from the I/O row to $\storearr[a_k]$ that avoids all previously stored loads.
Robots are positioned either under stored loads or at their initial cells (for robots that have not yet picked up a load). Therefore, the only obstacles on $\pi$ may be robots still at their initial cells.
If such a robot-obstacle $r$ exists, assume that it is the only one by the choice of the initial robot configuration. Therefore, path planning will succeed for $r$ as in the base case.
If there are no robot-obstacles on $\pi$, then there exists a robot that can reach the I/O row to pick up and store $a_k$ via $\pi$.
\end{proof}

For retrieval, we follow a similar approach to prove completeness.

\begin{restatable}{theorem}{retrieval} \label{thm:retrieval}
The retrieval algorithm is complete.
\end{restatable}
\begin{proof}
The proof is by induction on $k$ that load $k$ can be retrieved.
For the base case, consider load 1, the first to depart.
First, a robot must claim load 1.
Indeed, as idle robots are only obstructed by other robots, some robot $r$ can claim load $1$ and proceed to $\storearr[1]$.
Next, we need to show that $r$ can drop load 1 off by successfully planning two path segments:
(i) from $\storearr[1]$ to the I/O row and (ii) then back to $W$.
By the definition of the arrangement \storearr,
there exists a path from $\storearr[1]$ to the I/O row (more specifically, $\storearr[1]$ must be adjacent to the I/O row).
Path planning for $r$ must also succeed for the second path segment: there is a load under which $r$ can wait, which can be either the next unclaimed load or an arbitrary one.
Other robots cannot interfere with the two path segments, as no drop-off paths are planned for them until a drop-off for load 1 is queued. %

Assume loads $1, \ldots, k-1$ have been retrieved and consider load $k$.
As in the storage case, let us assume that all the robots become idle in \storage before $k$ is retrieved, as otherwise there is nothing to show.
There exists a load-free path $\pi$ from $\storearr[k]$ to the I/O row as $\storearr$ satisfies $D$.
First, some robot $r$ can claim load $k$ and proceed to $\storearr[k]$, which is accessible to an idle robot due to the existence of $\pi$. %
We now assume that every idle robot is under a load (idle robots not under a load will be handled below), which means that $r$ can drop load $k$ off via $\pi$.

From here, two cases remain for the return path to \storage: in the first, $r$ returns to wait under some load.
Otherwise, there is no accessible load for $r$ to go under.
Assume the second case, as otherwise we are done.%

This case is presented in the full proof in~\cite[Appendix]{complete}.
\end{proof}

\vspace{-3pt}
\section{Experiments}
\vspace{-3pt}

This section evaluates the performance of the proposed prioritized planner integrated with the arrangement choices of the Storage and Retrieval with Minimum Relocations algorithms \cite{stormr25,aaai26}. 

The experiments are designed to validate three primary claims. \textit{Efficiency:} The algorithm effectively utilizes multiple robots to reduce the total makespan of storage and retrieval tasks, achieving near-linear speedup in the number of robots. \textit{Scalability:} The planner maintains low runtime per task, making it suitable for online, real-time operation. \textit{Solution quality:} The solutions exhibit low suboptimality when compared to a theoretically optimal (but computationally expensive) coupled planner. 

The experiments are executed on a variety of relocation-free storage arrangements and for $1$ to $C$ robots. To generate relocation-free storage arrangements, the experiments use the \stormr solver $(k=0)$ \cite{stormr25} and the \rstormr solver \cite{aaai26} for $k=0.4 \cdot C$ for a $k$-robust arrangement. 

Each instance randomizes an arrival sequence \inc and a departure sequence \dep. For experiments testing robust arrangements given $k$, we generate a $k$-perturbed sequence $\tilde{D}$. Experiments were run on Ubuntu 22.04.5 LTS with an Intel Xeon Gold 5220 CPU (2.20 GHz). A comprehensive listing of experimental results for additional configurations (grid sizes, $k$ values) is provided in~\cite[Appendix]{complete}.

\begin{figure}[t!]
    \centering
    
    \begin{subfigure}[b]{0.33\textwidth}
        \centering
        \includegraphics[width=\textwidth]{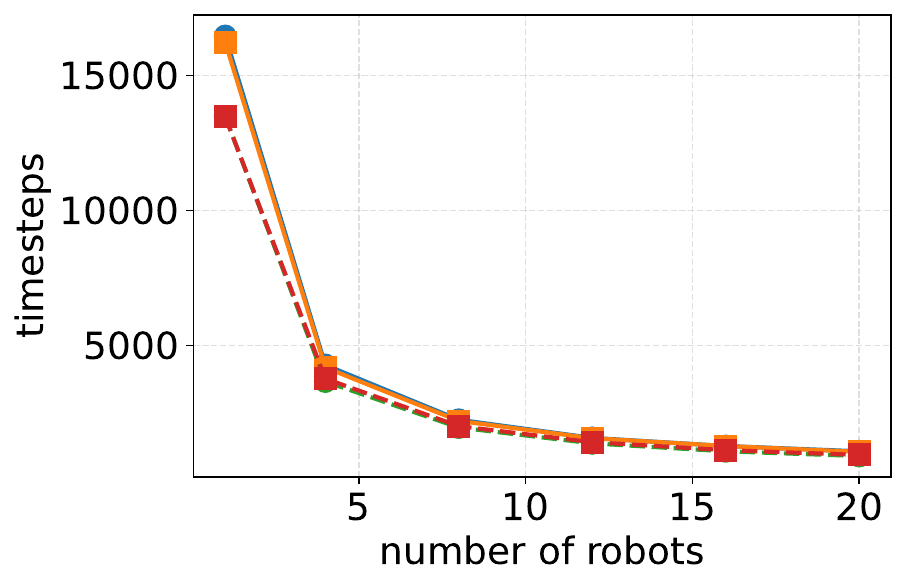}
        \caption{Makespan ($20 \times 20$)}
    \end{subfigure}
    \hfill
    \begin{subfigure}[b]{0.305\textwidth}
        \centering
        \includegraphics[width=\textwidth]{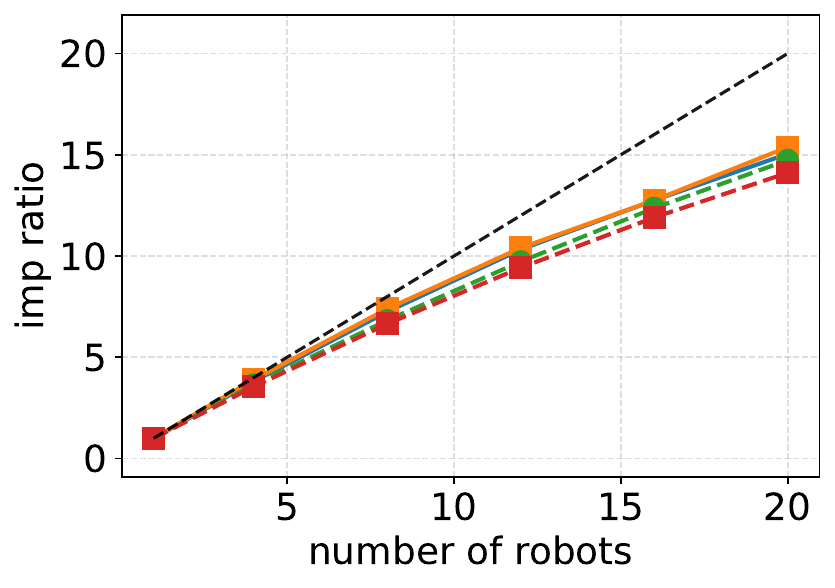}
        \caption{Imp. Ratio ($20 \times 20$)}
    \end{subfigure}
    \hfill
    \begin{subfigure}[b]{0.33\textwidth}
        \centering
        \includegraphics[width=\textwidth]{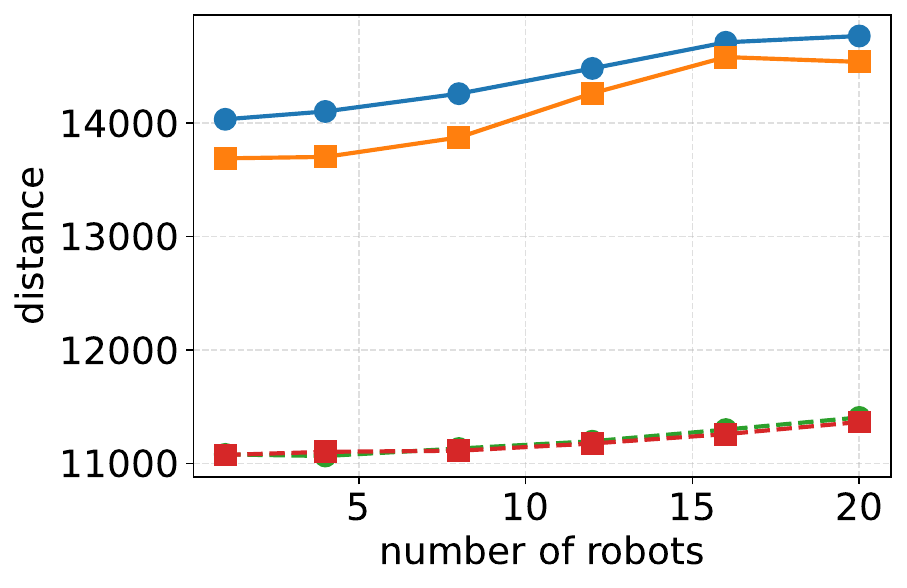}
        \caption{Distance ($20 \times 20$)}
    \end{subfigure}

    \vspace{0.5em} %

    \begin{subfigure}[b]{0.33\textwidth}
        \centering
        \includegraphics[width=\textwidth]{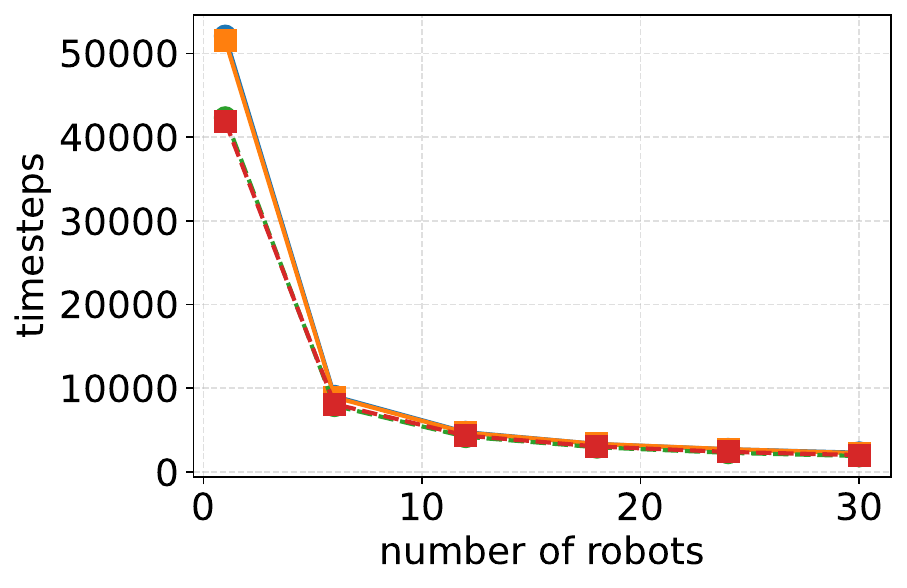}
        \caption{Makespan ($30 \times 30$)}
    \end{subfigure}
    \hfill
    \begin{subfigure}[b]{0.305\textwidth}
        \centering
        \includegraphics[width=\textwidth]{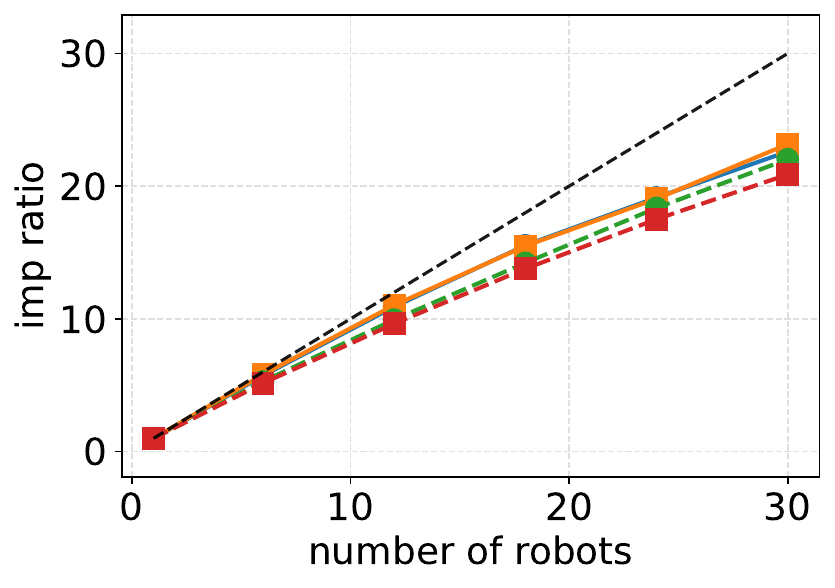}
        \caption{Imp. Ratio ($30 \times 30$)}
    \end{subfigure}
    \hfill
    \begin{subfigure}[b]{0.33\textwidth}
        \centering
        \includegraphics[width=\textwidth]{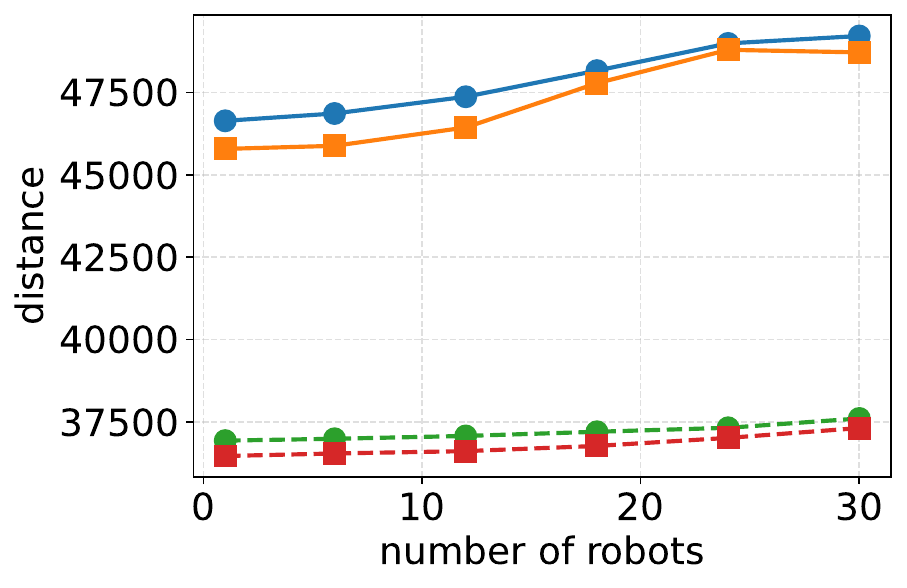}
        \caption{Distance ($30 \times 30$)}
    \end{subfigure}

     \includegraphics[width=0.8\textwidth]{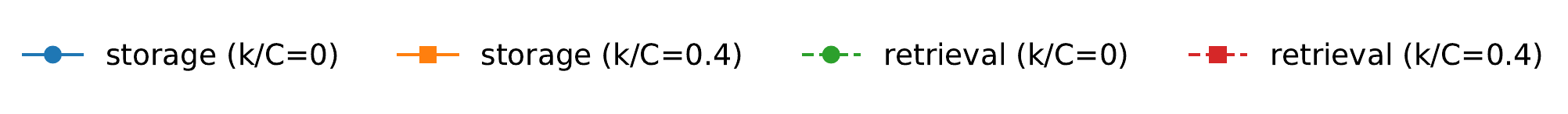}
    \vspace{-.2in}
    \caption{Performance comparison. Top row: $20\times 20$ grids; Bottom row: $30\times 30$ grids. Columns (left to right): Total Makespan, Makespan Improvement Ratio, and Total Distance traveled. Each data point averages 50 experimental trials.}
    \label{fig:experiment_plots}
\end{figure}

To quantify the performance gain provided by multiple robots relative to the system scale, consider the \textit{makespan improvement ratio} ($S_m$): the makespan improvement ratio is defined as $S_m = \frac{T_1}{T_m}$, where for a given grid size $\cols \times \cols$ and load count $n=\cols \cdot \cols$, let $T_1$ be the optimal makespan achieved by a single robot and $T_m$ be the makespan achieved by $m$ robots.

Figure~\ref{fig:experiment_plots} demonstrates that the prioritized planning approach achieves near-linear scalability relative to the number of robots ($m$). Figures~\ref{fig:experiment_plots}(a)(b) provide the absolute makespan and the improvement ratio. They demonstrate the significant decrease in makespan with the introduction of additional robots, which directly relates to the throughput efficiency of the storage system.
$S_m$ adheres closely to the ideal linear benchmark up to approximately $S_m=15$ for a $20 \times 20$ grid with $20$ robots.
This suggests that the system can support a moderate robot density before interference significantly impacts efficiency.

\textbf{Robustness.}
A key contribution of this work is the integration of \rstormr arrangements. These arrangements allow completeness under uncertain departure sequences that are revealed sequentially during retrieval. Notably, Fig.~\ref{fig:experiment_plots} reveals that the execution penalty for using robust arrangements ($k=0.4C$) compared to standard arrangements ($k=0$) is negligible (plots for the robust and the non-robust storage overlap).
This indicates that the robustness introduced by \rstormr to handle retrieval uncertainty does not hinder robot motion. 

\textbf{Coordination overhead.}
Figure~\ref{fig:experiment_plots}(c)(f) shows the average total distance traveled by robots. While the total distance naturally increases with $m$ due to conflict avoidance maneuvers, the slope is remarkably shallow (less than $5\%$ increase for $m=20$ compared to the single robot case).
This indicates that there is only a small sacrifice due to multi-robot coordination and collision avoidance. 

\begin{figure}[t]
    \centering
    \begin{subfigure}[b]{0.4\textwidth}
        \centering
        \includegraphics[width=\textwidth]{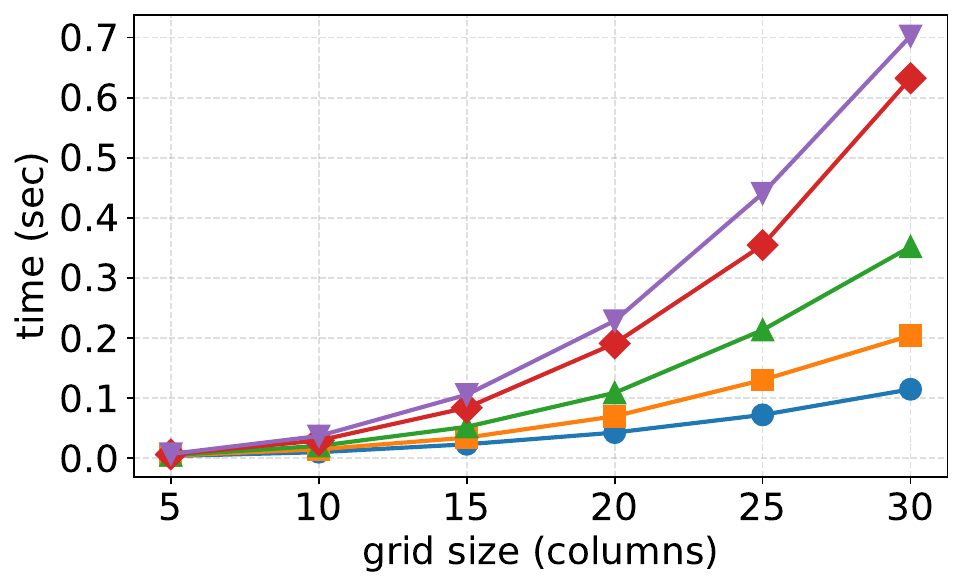}
        \caption{Storage runtime}
    \end{subfigure}
    \quad %
    \begin{subfigure}[b]{0.4\textwidth}
        \centering
        \includegraphics[width=\textwidth]{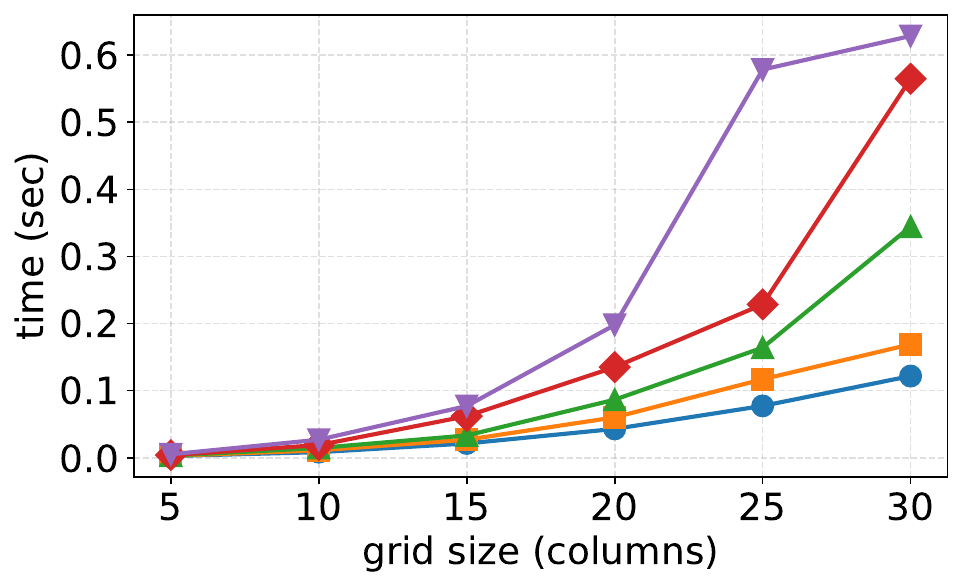}
        \caption{Retrieval runtime}
    \end{subfigure}
    \begin{subfigure}[b]{0.99\textwidth}
        \centering
        \includegraphics[width=\textwidth]{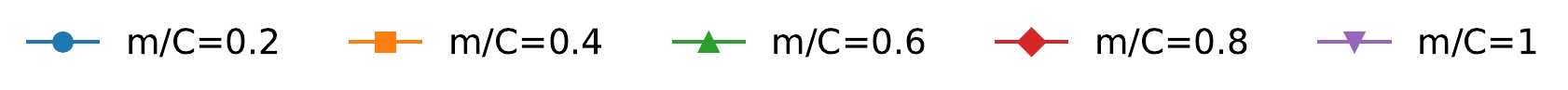}
    \end{subfigure}
    \vspace{-0.1in}
    \caption{Mean runtime per load
    for storage (left) and retrieval (right) as the square grid size increases (the $x$-axis corresponds to the number of columns \cols).
    Each plotted line corresponds to a different number of robots $m$ as a proportion of \cols.}
    \label{fig:runtime_plots}
\end{figure}
\textbf{Runtime.}
Figure~\ref{fig:runtime_plots} shows the runtime per load, which is the average planning time it takes to successfully calculate a path to store or retrieve a load as the grid size increases and for different numbers of robots.
There is a reasonable increase in planning runtime as the size of the grid increases as well as the number of robots increases. Nevertheless, the runtime always remains in the sub-second territory for the experiments performed, which can easily be accommodated by a real-world system in this space and allow the execution of these plans without any robot waiting for path planning.

\begin{table}[t]
\centering
\caption{Makespan suboptimality ratio for 2 robots. (20 trials per entry)}
\vspace{-0.1in}
\label{tab:storage-retrieval}
\begin{minipage}{0.48\linewidth}
    \centering
    \subcaption{Storage}
    \begin{tabularx}{\textwidth}{lXXXX}
        \toprule
        Grid\;\;\;\; & 3$\times$3 & 4$\times$4 & 5$\times$5 & 6$\times$6 \\
        \midrule
        Ratio & 1.091 & 1.162 & 1.210 & 1.176 \\
        \bottomrule
    \end{tabularx}
\end{minipage}
\hfill
\begin{minipage}{0.48\linewidth}
    \centering
    \subcaption{Retrieval}
    \begin{tabularx}{\textwidth}{lXXXX}
        \toprule
        Grid\;\;\;\;  & 3$\times$3 & 4$\times$4 & 5$\times$5 & 6$\times$6 \\
        \midrule
        Ratio & 1.086 & 1.108 & 1.129 & 1.111 \\
        \bottomrule
    \end{tabularx}
\end{minipage}
\end{table}

\textbf{Makespan optimality.}
To evaluate the solution quality of the prioritized planning algorithm, this section compares against an optimal coupled, centralized planner based on A* search in the joint configuration space of two robots, which is executed for both robots every time a robot is assigned a new storage or retrieval task.
To ensure computationally feasible runtimes for the optimal solver, we decompose the problem into batches of size $C$. That is, we run the storage (or retrieval respectively) with the goal of storing (or retrieving) a segment of $C$ loads in \inc or \dep. Then, the robots are placed back to their spots at the end of the batch and the solver is rerun for the next batch. Each trial sums up the makespan of each batch in \inc or \dep.
Given the computational explosion of the A* search, the method was only executed for these small multi-robot instances due to memory constraints. See~\cite[Appendix]{complete} for details about this solver.

Table~\ref{tab:storage-retrieval} shows that for small instances where the optimal coupled solver is tractable, the proposed algorithm exhibits a suboptimality ratio ranging from $1.09$ to $1.21$. A portion of this suboptimality comes from the optimal solver taking advantage of the conveyor model, which allows loads to be stored and retrieved slightly out of order as long as no blockages occur. The prioritized approach uses the fixed arrival and departure sequences for completeness guarantees. 

\section{Conclusion}

This paper presents a comprehensive algorithmic framework for multi-robot ordered storage and retrieval in maximum-capacity environments. By bridging the gap between the zero-relocation guarantees of \stormr and \rstormr arrangements and online prioritized planning, it demonstrates that high-density warehouses can achieve a scalable high-throughput two-step, storage and retrieval, process even at maximum capacity.

The theoretical analysis confirms that the specific invariants of these storage arrangements provide sufficient clearance to guarantee completeness for sequential prioritized planning, avoiding the deadlocks common in prioritized MAPF. Experimentally, we show that this approach achieves near-linear makespan improvement up to $m =C$. A critical finding is the negligible execution cost of robustness: the \rstormr layouts, which protect against departure sequence uncertainty, can be executed just as efficiently as the \stormr layouts. While the approach incurs a suboptimality gap compared to coupled search, it maintains scalability for large instances ($30 \times 30$ grids with 30 robots).

Future work could focus on narrowing the optimality gap by investigating other MAPF techniques, such as PIBT and LaCAM \cite{PiBT,lacam}.
Additionally, investigations into \stormr and \rstormr could reveal better arrangements that tailor to multiple robots.
If departure uncertainty becomes too high, target loads may become blocked.
This can motivate the development of a multi-robot blocked retrieval problem. 
Beyond algorithmic variations, this work can be extended to different environments, such as those that allow multiple-side access as well as static obstacles within the storage space.

\bibliographystyle{splncs04}
\bibliography{refs}
\newpage           
\appendix          

\section{Experimental Results}
\begin{longtable}{ccc|cccc|cccc}
\caption{Performance metrics for storage and retrieval phases varying grid size, $k$, and robot counts. The table reports makespan $T_m$, makespan improvement ratio $T_1/T_m$, runtime per load $t/n$, and distance all averaged of 50 trials per entry.}
\label{tab:multi_robot_results} \\

\toprule
\multirow{2}{*}{Grid} & \multirow{2}{*}{$k$} & Robots & \multicolumn{4}{c|}{Storage} & \multicolumn{4}{c}{Retrieval} \\
\cmidrule(lr){4-7} \cmidrule(lr){8-11}
& & ($m$) & $T_m$ & $\frac{T_1}{T_m}$ & $\frac{t}{n}$ & Dist. & $T_m$ & $\frac{T_1}{T_m}$ & $\frac{t}{n}$ & Dist. \\
\midrule
\endfirsthead

\multicolumn{11}{c}{{\bfseries \tablename\ \thetable{} -- continued from previous page}} \\
\toprule
\multirow{2}{*}{Grid} & \multirow{2}{*}{$k$} & Robots & \multicolumn{4}{c|}{Storage} & \multicolumn{4}{c}{Retrieval} \\
\cmidrule(lr){4-7} \cmidrule(lr){8-11}
& & ($m$) & $T_m$ & $\frac{T_1}{T_m}$ & $\frac{t}{n}$ & Dist. & $T_m$ & $\frac{T_1}{T_m}$ & $\frac{t}{n}$ & Dist. \\
\midrule
\endhead

\midrule
\multicolumn{11}{r}{{Continued on next page}} \\
\bottomrule
\endfoot

\bottomrule
\endlastfoot

\footnotesize 

    & \multirow[t]{5}{*}{0} & 2  & 1248.16 & 1.95 & 0.01 & 1825.08 & 1064.30 & 1.91 & 0.01 & 1440.00 \\
    &                    & 4  & 656.78  & 3.70 & 0.02 & 1842.60 & 568.24  & 3.58 & 0.01 & 1447.56 \\
    &                    & 6  & 461.02  & 5.27 & 0.02 & 1861.38 & 404.58  & 5.03 & 0.02 & 1462.66 \\
    &                    & 8  & 369.08  & 6.58 & 0.03 & 1881.50 & 325.06  & 6.26 & 0.03 & 1478.72 \\
    &                    & 10 & 312.02  & 7.79 & 0.03 & 1886.80 & 282.26  & 7.21 & 0.04 & 1508.44 \\
    \cmidrule{2-11}
    & \multirow[t]{5}{*}{2} & 2  & 1217.06 & 1.96 & 0.01 & 1763.78 & 1082.06 & 1.90 & 0.01 & 1449.72 \\
    &                    & 4  & 641.60  & 3.71 & 0.01 & 1779.62 & 582.44  & 3.52 & 0.01 & 1456.56 \\
$10 \times 10$ &         & 6  & 454.74  & 5.24 & 0.02 & 1809.14 & 421.12  & 4.87 & 0.01 & 1472.38 \\
    &                    & 8  & 368.96  & 6.46 & 0.03 & 1849.42 & 339.26  & 6.05 & 0.01 & 1493.84 \\
    &                    & 10 & 313.42  & 7.60 & 0.04 & 1849.16 & 297.16  & 6.91 & 0.02 & 1522.84 \\
    \cmidrule{2-11}
    & \multirow[t]{5}{*}{4} & 2  & 1215.92 & 1.96 & 0.01 & 1762.20 & 1079.34 & 1.91 & 0.01 & 1453.42 \\
    &                    & 4  & 641.56  & 3.71 & 0.01 & 1779.52 & 583.98  & 3.52 & 0.01 & 1463.66 \\
    &                    & 6  & 457.16  & 5.20 & 0.02 & 1819.54 & 415.04  & 4.95 & 0.01 & 1472.12 \\
    &                    & 8  & 369.18  & 6.45 & 0.03 & 1847.48 & 333.12  & 6.17 & 0.02 & 1490.78 \\
    &                    & 10 & 310.14  & 7.67 & 0.04 & 1836.56 & 288.88  & 7.12 & 0.02 & 1518.96 \\
\midrule

    & \multirow[t]{5}{*}{0} & 3  & 2544.76 & 2.90 & 0.03 & 6024.98 & 2164.22 & 2.81 & 0.03 & 4736.66 \\
    &                    & 6  & 1343.10 & 5.49 & 0.04 & 6093.16 & 1159.20 & 5.24 & 0.03 & 4756.20 \\
    &                    & 9  & 941.84  & 7.82 & 0.05 & 6182.66 & 816.80  & 7.43 & 0.04 & 4790.66 \\
    &                    & 12 & 757.18  & 9.73 & 0.08 & 6267.14 & 646.28  & 9.39 & 0.08 & 4848.72 \\
    &                    & 15 & 643.46  & 11.45& 0.10 & 6298.60 & 550.36  & 11.03& 0.11 & 4898.60 \\
    \cmidrule{2-11}
    & \multirow[t]{5}{*}{3} & 3  & 2490.62 & 2.90 & 0.02 & 5855.36 & 2203.68 & 2.72 & 0.02 & 4666.98 \\
    &                    & 6  & 1308.98 & 5.52 & 0.03 & 5913.76 & 1197.72 & 5.01 & 0.02 & 4687.52 \\
$15 \times 15$ &         & 9  & 936.22  & 7.72 & 0.05 & 6064.60 & 853.42  & 7.03 & 0.03 & 4736.94 \\
    &                    & 12 & 753.30  & 9.60 & 0.08 & 6147.78 & 690.90  & 8.69 & 0.05 & 4792.98 \\
    &                    & 15 & 630.98  & 11.46& 0.11 & 6183.90 & 596.06  & 10.07& 0.06 & 4859.48 \\
    \cmidrule{2-11}
    & \multirow[t]{5}{*}{6} & 3  & 2484.02 & 2.92 & 0.02 & 5848.38 & 2211.74 & 2.76 & 0.02 & 4750.58 \\
    &                    & 6  & 1305.68 & 5.55 & 0.03 & 5911.20 & 1186.28 & 5.15 & 0.03 & 4751.50 \\
    &                    & 9  & 932.24  & 7.77 & 0.05 & 6075.76 & 839.42  & 7.27 & 0.03 & 4784.66 \\
    &                    & 12 & 755.02  & 9.59 & 0.09 & 6202.26 & 670.36  & 9.11 & 0.05 & 4847.66 \\
    &                    & 15 & 631.22  & 11.47& 0.11 & 6185.78 & 573.62  & 10.64& 0.06 & 4901.24 \\
\midrule

    & \multirow[t]{5}{*}{0} & 4  & 4279.02 & 3.85 & 0.05 & 14102.66 & 3666.68 & 3.68 & 0.05 & 11067.26 \\
    &                    & 8  & 2249.80 & 7.32 & 0.07 & 14259.30 & 1963.18 & 6.87 & 0.07 & 11133.18 \\
    &                    & 12 & 1579.62 & 10.42& 0.10 & 14481.82 & 1377.78 & 9.79 & 0.13 & 11197.10 \\
    &                    & 16 & 1278.64 & 12.87& 0.17 & 14711.82 & 1083.46 & 12.44& 0.22 & 11300.12 \\
    &                    & 20 & 1083.90 & 15.19& 0.21 & 14767.46 & 909.42  & 14.82& 0.33 & 11406.84 \\
    \cmidrule{2-11}
    & \multirow[t]{5}{*}{4} & 4  & 4191.68 & 3.87 & 0.04 & 13735.94 & 3761.70 & 3.52 & 0.04 & 10839.24 \\
    &                    & 8  & 2212.18 & 7.34 & 0.07 & 13910.62 & 2041.42 & 6.48 & 0.06 & 10923.46 \\
$20 \times 20$ &         & 12 & 1579.36 & 10.28& 0.11 & 14289.42 & 1462.12 & 9.04 & 0.07 & 11015.42 \\
    &                    & 16 & 1284.56 & 12.64& 0.20 & 14587.32 & 1171.96 & 11.28& 0.10 & 11131.72 \\
    &                    & 20 & 1069.88 & 15.17& 0.24 & 14596.62 & 993.34  & 13.31& 0.14 & 11251.76 \\
    \cmidrule{2-11}
    & \multirow[t]{5}{*}{8} & 4  & 4185.36 & 3.88 & 0.04 & 13702.20 & 3773.20 & 3.57 & 0.04 & 11104.10 \\
    &                    & 8  & 2205.94 & 7.35 & 0.07 & 13873.84 & 2014.32 & 6.69 & 0.05 & 11113.72 \\
    &                    & 12 & 1567.54 & 10.35& 0.11 & 14262.68 & 1421.76 & 9.48 & 0.06 & 11177.42 \\
    &                    & 16 & 1278.66 & 12.69& 0.19 & 14581.86 & 1125.38 & 11.98& 0.08 & 11259.78 \\
    &                    & 20 & 1060.70 & 15.30& 0.23 & 14539.18 & 949.56  & 14.19& 0.12 & 11366.64 \\
\midrule

    & \multirow[t]{5}{*}{0} & 5  & 6440.94 & 4.80 & 0.08 & 27295.38 & 5579.32 & 4.52 & 0.10 & 21475.06 \\
    &                    & 10 & 3385.04 & 9.14 & 0.12 & 27607.42 & 2974.60 & 8.48 & 0.16 & 21556.54 \\
    &                    & 15 & 2376.20 & 13.02& 0.19 & 28049.42 & 2083.24 & 12.11& 0.21 & 21679.50 \\
    &                    & 20 & 1927.18 & 16.06& 0.32 & 28511.14 & 1625.32 & 15.52& 0.38 & 21753.66 \\
    &                    & 25 & 1632.36 & 18.96& 0.40 & 28657.10 & 1364.64 & 18.49& 0.78 & 21987.68 \\
    \cmidrule{2-11}
    & \multirow[t]{5}{*}{5} & 5  & 6344.88 & 4.82 & 0.07 & 26695.08 & 5710.86 & 4.29 & 0.06 & 20739.52 \\
    &                    & 10 & 3352.28 & 9.13 & 0.14 & 27057.08 & 3096.76 & 7.90 & 0.08 & 20894.58 \\
$25 \times 25$ &         & 15 & 2397.14 & 12.76& 0.24 & 27845.02 & 2194.28 & 11.16& 0.13 & 21067.52 \\
    &                    & 20 & 1939.22 & 15.78& 0.39 & 28343.90 & 1753.94 & 13.96& 0.13 & 21230.36 \\
    &                    & 25 & 1617.78 & 18.91& 0.49 & 28549.90 & 1479.44 & 16.55& 0.70 & 21117.32 \\
    \cmidrule{2-11}
    & \multirow[t]{5}{*}{10} & 5  & 6333.78 & 4.84 & 0.07 & 26684.48 & 5715.42 & 4.38 & 0.07 & 21302.74 \\
    &                    & 10 & 3333.12 & 9.19 & 0.13 & 26981.10 & 3052.00 & 8.20 & 0.10 & 21361.28 \\
    &                    & 15 & 2370.84 & 12.92& 0.21 & 27687.82 & 2148.22 & 11.65& 0.14 & 21450.26 \\
    &                    & 20 & 1916.00 & 15.98& 0.35 & 28222.88 & 1690.86 & 14.80& 0.17 & 21603.88 \\
    &                    & 25 & 1586.22 & 19.31& 0.43 & 28205.54 & 1422.68 & 17.59& 0.28 & 21826.48 \\
\midrule

    & \multirow[t]{5}{*}{0} & 6  & 9037.86 & 5.76 & 0.12 & 46866.06 & 7906.92 & 5.35 & 0.15 & 36985.38 \\
    &                    & 12 & 4741.84 & 10.99& 0.18 & 47374.56 & 4198.40 & 10.08& 0.22 & 37071.42 \\
    &                    & 18 & 3328.48 & 15.65& 0.32 & 48168.60 & 2936.28 & 14.42& 0.65 & 37198.42 \\
    &                    & 24 & 2702.14 & 19.28& 0.54 & 48995.06 & 2273.46 & 18.62& 1.15 & 37319.12 \\
    &                    & 30 & 2287.92 & 22.77& 0.56 & 49220.14 & 1895.33 & 22.33& 1.31 & 37602.25 \\
    \cmidrule{2-11}
    & \multirow[t]{5}{*}{6} & 6  & 8914.50 & 5.78 & 0.11 & 45914.20 & 8152.98 & 5.03 & 0.10 & 35710.22 \\
    &                    & 12 & 4710.70 & 10.94& 0.22 & 46526.30 & 4414.92 & 9.29 & 0.14 & 35955.14 \\
$30 \times 30$ &         & 18 & 3365.44 & 15.31& 0.38 & 47877.88 & 3127.48 & 13.12& 0.17 & 36187.24 \\
    &                    & 24 & 2741.58 & 18.79& 0.71 & 48934.84 & 2489.34 & 16.48& 0.28 & 36414.38 \\
    &                    & 30 & 2269.24 & 22.70& 0.83 & 49104.20 & 2110.82 & 19.44& 0.32 & 36721.08 \\
    \cmidrule{2-11}
    & \multirow[t]{5}{*}{12} & 6  & 8912.64 & 5.79 & 0.11 & 45884.00 & 8104.24 & 5.17 & 0.11 & 36539.04 \\
    &                    & 12 & 4688.54 & 11.00& 0.21 & 46441.36 & 4331.82 & 9.66 & 0.15 & 36610.00 \\
    &                    & 18 & 3339.64 & 15.44& 0.36 & 47785.64 & 3029.48 & 13.82& 0.21 & 36767.94 \\
    &                    & 24 & 2715.76 & 18.99& 0.65 & 48800.34 & 2384.30 & 17.56& 0.27 & 37012.88 \\
    &                    & 30 & 2233.52 & 23.09& 0.72 & 48722.08 & 1999.06 & 20.94& 0.29 & 37311.16 \\

\end{longtable}

\section{Optimal Solver}
\paragraph{Joint State Space.}
The system state at timestep $t$ is defined as a tuple $S_t = \langle {R}_t, {L}_t \rangle$. Here, $L_t$ represents the current grid configuration (load locations), and $R_t = \{r_1, \dots, r_m\}$ represents the joint configuration of the $m$ robots. Each robot state $r_i$ is a tuple $(v, \theta, \beta)$, where $v$ is its position, $\theta$ is its direction, and $\beta\in \{0,1\}$ indicates whether the robot is carrying a load. 
\paragraph{Transitions.}
Transitions between states represent a single discrete timestep where each robot performs a primitive action $u \in \{\text{Move},\allowbreak \text{Turn},\allowbreak \text{Wait},\allowbreak \text{Pick},\allowbreak \text{Drop}\}$. A transition $S_t \rightarrow S_{t+1}$ is valid if and only if it corresponds with a valid robot or conveyor action. Drop actions are restricted until robots reach the load's goal state. In the storage phase, loads can only be dropped at their storage location, and in retrieval, loads can only be dropped on the I/O row. That is, drop-off and pick-up relays to store or retrieve a package are disallowed. To reduce the branching factor, any timestep where all robots wait with no conveyor action is disallowed. The conveyor follows deterministic rules based on batch mechanics (storage) or the departure schedule (retrieval), contributing zero additional branching factor to the search. Additionally, loads in the storage grid that are not part of the arrival or departure batch are considered immovable obstacles. 

\subsection{Heuristic}

Let $U(s)$ be the unfinished loads that are not currently carried in state $s$, and let $C(s)$ be the unfinished loads that are currently carried. For $\ell \in U(s)$, let $p_\ell$ be its pickup cell. For storage, this is its current or designated future pickup cell on the I/O row. For retrieval, this is its current cell in the grid.

For each unfinished load $\ell$, let $G_\ell$ be the set of oriented robot configurations at which $\ell$ may legally be dropped. For storage, $G_\ell$ consists of the four orientations at the assigned cell $A[\ell]$. For retrieval, it consists of the four orientations at every I/O-row cell. This uses the restriction that a load may be dropped only at a legal destination.

Let $Q(p)$ be the four oriented configurations at cell $p$. Let $q_r(s)$ denote the current oriented configuration of robot $r$ in state $s$ and $r(\ell)$ the robot that is currently carrying $\ell$. A robot that is not currently carrying is called empty. Let $d(X, Y)$ be the minimum number of move and turn actions required to move one robot from some configuration in set $X$ to some configuration in set $Y$, assuming an empty grid. 

For each $\ell \in U(s)$, define
\[
O_\ell(s) = \{ q_r(s) : r \text{ is currently free} \} \cup \bigcup_{\substack{\ell' \in U(s) \cup C(s) \\ \ell' \neq \ell}} G_{\ell'}.
\]

A robot that eventually picks up $\ell$ starts its empty-approach either from its current configuration if it is currently empty or from the destination where it previously dropped another load if it is currently carrying. Therefore, define
\[
W_{\text{pickup}}(s) = \sum_{\ell \in U(s)} d(O_\ell(s), Q(p_\ell)).
\]

The mandatory work for carrying and completing the remaining loads is
\[
W_{\text{carry}}(s) = \sum_{\ell \in U(s)} \left( 2 + d(Q(p_\ell), G_\ell) \right) + \sum_{\ell \in C(s)} \left( 1 + d(q_{r(\ell)}(s), G_\ell) \right),
\]
where $r(\ell)$ is the robot currently carrying $\ell$. An uncarried load is charged one pickup and one drop. A carried load is charged only its remaining motion and one drop.

The heuristic is
\[
h(s) =  \frac{W_{\text{pickup}}(s) + W_{\text{carry}}(s)}{m} .
\]

\subsection{Proof of admissibility}

\begin{theorem}
The heuristic $h$ never exceeds the minimum remaining makespan.
\end{theorem}

\begin{proof}
Fix a state $s$, and consider any feasible completion $\sigma$ of makespan $T$. Let $N_\sigma$ be the total number of non-wait robot actions in $\sigma$.

Consider an uncarried load $\ell \in U(s)$. Its serving robot must pick it up, carry it from $p_\ell$ to a legal destination, and drop it. This requires at least
\[
2 + d(Q(p_\ell), G_\ell)
\]
non-wait actions.

Consider instead a carried load $\ell \in C(s)$. Its current robot must move from its present configuration to a legal destination and drop it. This requires at least
\[
1 + d(q_{r(\ell)}(s), G_\ell)
\]
non-wait actions. Thus, the completion contains at least $W_{\text{carry}}(s)$ carry, pickup, and drop actions.

Now fix an uncarried load $\ell \in U(s)$, and inspect the robot that eventually picks it up. If that robot has not delivered another load since state $s$, its empty-approach begins at its current empty configuration. Otherwise, its empty-approach begins immediately after its most recent earlier drop, at a configuration in $G_{\ell'}$ for some $\ell' \neq \ell$. In both cases, the approach begins in $O_\ell(s)$ and ends in $Q(p_\ell)$. It therefore contains at least
\[
d(O_\ell(s), Q(p_\ell))
\]
move and turn actions. Summing over all uncarried loads gives at least $W_{\text{pickup}}(s)$ empty-approach actions.

The empty-approach portions and the load-delivery portions are pairwise action-disjoint. For each robot, they occur in alternating order: approach a load, pick it up, deliver it, and then approach another load. Consequently,
\[
N_\sigma \ge W_{\text{pickup}}(s) + W_{\text{carry}}(s).
\]

Because each joint timestep contains at most one non-wait action per robot,
\[
N_\sigma \le mT.
\]

Therefore,
\begin{align*}
W_{\text{pickup}}(s) + W_{\text{carry}}(s) &\le mT, \\[1ex]
h(s) = \frac{W_{\text{pickup}}(s) + W_{\text{carry}}(s)}{m} &\le T.
\end{align*}
This holds for every feasible completion, so $h(s) \le h^*(s)$.
\end{proof}

\section{Full Completeness Proof}
\retrieval*
\label{app:extended_proofs}
\begin{proof}
The proof is by induction on $k$ that load $k$ can be retrieved.
For the base case, consider load 1, the first to depart.
First, a robot must claim load 1.
Indeed, as idle robots are only obstructed by other robots, some robot $r$ can claim load $1$ and proceed to $\storearr[1]$.
Next, we need to show that $r$ can drop load 1 off by successfully planning two path segments:
(i) from $\storearr[1]$ to the I/O row and (ii) then back to $W$.
By the definition of the arrangement \storearr,
there exists a path from $\storearr[1]$ to the I/O row (more specifically, $\storearr[1]$ must be adjacent to the I/O row).
Path planning for $r$ must also succeed for the second path segment: there is a load under which $r$ can wait, which can be either the next unclaimed load or an arbitrary one.
Other robots cannot interfere with the two path segments, as no drop-off paths are planned for them until a drop-off for load 1 is queued. %

Assume loads $1, \ldots, k-1$ have been retrieved and consider load $k$.
As in the storage case, let us assume that all the robots become idle in \storage before $k$ is retrieved, as otherwise there is nothing to show.
There exists a load-free path $\pi$ from $\storearr[k]$ to the I/O row as $\storearr$ satisfies $D$.
First, some robot $r$ can claim load $k$ and proceed to $\storearr[k]$, which is accessible to an idle robot due to the existence of $\pi$. %
We now assume that every idle robot is under a load (idle robots not under a load will be handled below), which means that $r$ can drop load $k$ off via $\pi$.

From here, two cases remain for the return path to \storage: in the first, $r$ returns to wait under some load.
Otherwise, there is no accessible load for $r$ to go under.
Assume the second case, as otherwise we are done.%
We now show that there exists a cell $c$ that $r$ can reach from the I/O row and remain idle at such that no subsequent retrieval paths are blocked.
We have at most $\cols-1$ other robots idle in \storage.
Let $L$ denote the set of loads that are above these robots.
For $\ell \in L$, consider the shortest eventual retrieval path $\pi_{\ell}$ from $\storearr[\ell]$ to the I/O row, which is guaranteed to exist.
It suffices to ensure that $c$ is not located on any of the $\pi_{\ell}$'s since any retrieval path for a $\ell' \notin L$ must traverse a cell occupied by a load in $L$.
This is true because otherwise $\ell'$ has a robot-free path to the I/O row, which $r$ could use to go under $\ell'$---a contradiction.

Each $\pi_{\ell}$ exits \storage via a certain column, resulting in at most $\cols -1$ such exit columns.
We can therefore pick a non-exit column $K$. 
We argue that $K$ is empty and that the backmost cell on $K$ is a valid choice for $c$:
Suppose, for a contradiction, that there is a load $\ell'$ on $K$, chosen to be the bottommost load on $K$.
Then, we have $\ell' \in L$ (as otherwise $r$ could go under it), and because there are no loads below it, $\pi_{\ell'}$ must exit \storage via $K$, which contradicts the choice of $K$.
Therefore, $K$ is empty, which means that no $\pi_{\ell}$ passes through $K$:
Any $\pi_{\ell}$ passing through $K$ would also have to exit \storage via $K$ as that would be the shortest path, which again contradicts the choice of $K$.
Therefore, to summarize, we have an empty cell $c$ that $r$ can go to without blocking subsequent retrievals, which concludes this case.
Finally, we remark that for subsequent retrievals, we consider $K$ as one of the at most $C-1$ exit columns if $r$ is still idle on it, which always allows finding another empty column using the same arguments.
\end{proof}
 
\end{document}